\documentclass{article}

\usepackage[preprint]{neurips_2026}
\usepackage[utf8]{inputenc}
\usepackage[T1]{fontenc}
\usepackage{hyperref}
\usepackage{url}
\usepackage{booktabs}
\usepackage{amsfonts}
\usepackage{nicefrac}
\usepackage{microtype}
\usepackage{xcolor}

\usepackage{amsthm}
\usepackage{amsmath}
\usepackage{graphicx}
\usepackage{subcaption}
\usepackage{multirow}
\usepackage{enumitem}

\newtheorem{proposition}{Proposition}

\newcommand{\Pmass}{P_{\mathrm{mass}}}

\newcommand{\Sc}{S_c}
\newcommand{\Scstar}{S_{c^*}}
\newcommand{\Spread}{\mathrm{Spread}}

\title{Hallucination as Commitment Failure:\\Larger LLMs Misfire Despite Knowing the Answer}

\author{%
  Jewon Yeom$^{1}$ \quad Jaewon Sok$^{2}$ \quad Heejun Kim$^{3}$ \\
  \bfseries Seonghyeon Park$^{4}$ \quad Jeongjae Park$^{1}$ \quad Taesup Kim$^{1,}$\thanks{Corresponding author.} \\
  $^{1}$Graduate School of Data Science, Seoul National University \\
  $^{2}$Department of Rural Systems Engineering, Seoul National University \\
  $^{3}$Electrical Engineering and Computer Science, Gwangju Institute of Science and Technology \\
  $^{4}$Department of Aerospace Engineering, Seoul National University
}

\begin{document}
\maketitle

\begin{abstract}
Hallucination is often viewed as a direct consequence of missing knowledge: a model answers incorrectly when the correct answer is absent from its generation-time distribution, and correctly when it is present. We test this assumption by introducing a semantic notion of answer availability that aggregates token-level variants expressing the same answer concept, and asks whether the correct concept is already available at the moment the model commits to an answer. Across Qwen and Llama models from 0.8B to 72B in both Instruct and Base variants, 16--47\% of Instruct hallucinations occur with substantial probability mass already on the correct concept, and the rate rises monotonically with scale. Comparing such failures against correct generations with matched semantic support, the distinguishing factor is not whether the correct concept is represented, but how its probability is distributed: correct generations concentrate mass on a single surface form, hallucinations disperse it across alternatives. The same sharpening asymmetry extends across multi-token generation and is detectable in pre-generation hidden states. Together, these results identify a single mechanism: instruction tuning sharpens answer commitment with scale, making helpfulness and confident hallucination two consequences of the same underlying disposition.
\end{abstract}

%=============================================================================
\section{Introduction}
\label{sec:intro}
%=============================================================================

Large language models (LLMs) frequently produce fluent but factually incorrect outputs---hallucinations---that undermine their reliability in safety-critical applications \citep{ji2023survey}. A natural first question is \emph{where in the generation trajectory} a hallucination is decided. If hallucination emerges everywhere, post-hoc analysis of full sequences is necessary; if it localizes at specific steps, both detection and intervention should target those steps.

Recent work in the reasoning literature suggests sharp localization. Inspecting token-level entropy $H(y_t \mid Q, y_{<t})$ during greedy decoding reveals that entropy is highly non-uniform across the sequence: at most steps it is near zero (the next token is essentially deterministic) but at a small number of steps it spikes sharply. \citet{wang2025high} report that $\sim$20\% of tokens in chain-of-thought traces carry high entropy and act as ``forking tokens'' that determine reasoning paths; \citet{vassoyan2025critical} identify ``critical tokens'' as decision points where models are most error-prone. Figure~\ref{fig:motivation} shows the same phenomenon in a QA setting: an early spike fixes the domain of the answer (\texttt{Britain}), a later spike selects the answer entity (\texttt{Nicola}), and the steps in between are syntactic continuations following automatically.

\begin{figure}[t]
\centering
\includegraphics[width=\linewidth]{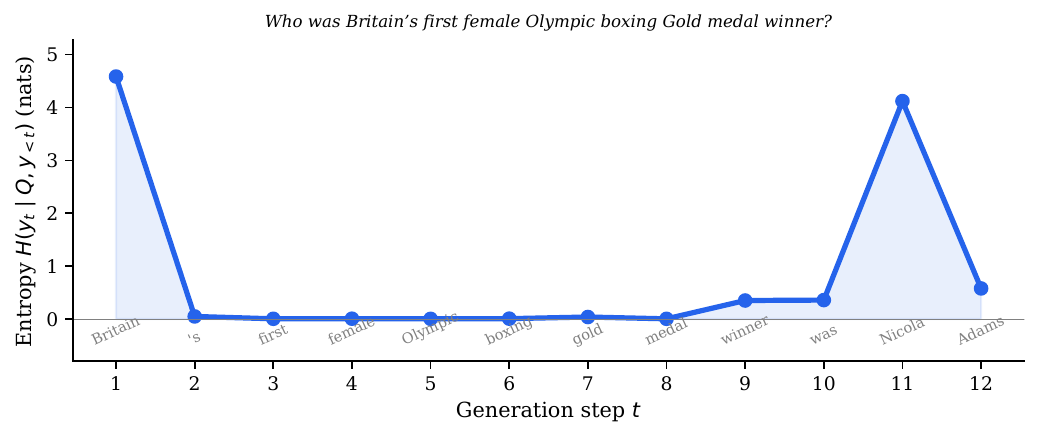}
\caption{Token entropy $H(y_t \mid Q, y_{<t})$ across a representative generation trajectory (Qwen3.5-9B Instruct). Entropy is near zero at most steps but spikes sharply at a small number of \emph{commitment steps}.}
\label{fig:motivation}
\vspace{-1.5em}
\end{figure}

A natural follow-up is whether entropy at these spikes is itself a hallucination signal. Existing work has established that it is not, in a stronger form than we will need: \citet{simhi2025trust} document hallucinations produced with high certainty even when the model demonstrably has the correct knowledge, \citet{liu2025delusions} formalize ``high-belief hallucinations'' as a distinct phenomenon from confidence-based detection, and the original semantic entropy work of \citet{farquhar2024detecting} explicitly notes that confidently wrong outputs are a separate phenomenon from the confabulation regime that semantic entropy targets. The literature's response has been to develop more sophisticated estimators \citep{kuhn2023semantic, ma2025semantic} or perturbation-based diagnostics \citep{simhi2025trust} that better separate hallucinated from non-hallucinated outputs. We pursue an orthogonal direction: rather than designing a better detector, we ask what the model's distribution is doing at the commitment step when the final answer is hallucinated, regardless of whether the wrong answer is emitted with high confidence or not.

We refer to the answer-emission step as the \emph{commitment step} $t_c$. In short-form QA with instruction-tuned models, $t_c = 1$ (§\ref{sec:eps_align}), so the prompt format fixes the commitment step at $t=1$ and lets us inspect the distribution at a single, known step. What it reveals is that ``high entropy'' has two structurally different sources: mass spread over genuinely different answers, and mass spread over different surface forms of the same answer (\texttt{Paris}, \texttt{ Paris}, \texttt{paris}, or \texttt{St}, \texttt{Saint}, \texttt{C} for \texttt{St.\ Basil's Cathedral}, Figure~\ref{fig:fragmentation}). To distinguish them, we define a \emph{concept} as the equivalence class of token completions denoting the same answer and introduce the \emph{per-step semantic probability mass} $\Pmass(t; c) = \sum_{v \in \Sc} P_\theta(v \mid Q, y_{<t})$, where $\Sc$ collects the first-token IDs of the concept's surface forms. With $\Sc$ built from ground-truth aliases, $\Pmass(t; c^*)$ is an analytical probe---requiring the answer concept as input---rather than a deployable detector, but precisely this property lets us ask whether the model put substantial mass on the right answer at the moment of commitment.

\begin{figure}[t]
\centering
\includegraphics[width=\linewidth]{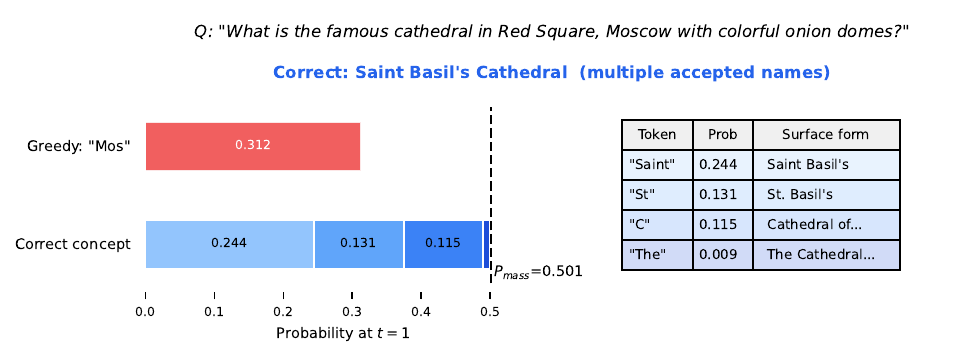}
\caption{Vocabulary fragmentation at the commitment step: the correct concept's mass (0.501 total) is split across \texttt{Saint} (0.244), \texttt{St} (0.115), and \texttt{C} (0.131); the greedy token is the competing \texttt{Mos} (0.312). Greedy decoding hides what the distribution says about the correct concept. This pattern is most pronounced in small and Base models; in large Instruct models, fragmentation collapses (§\ref{sec:fragmentation}) and commitment failures arise instead from a wrong concept's token being even sharper than the correct concept's collapsed mass.}
\label{fig:fragmentation}
\end{figure}

The headline finding is that across nine instruction-tuned Qwen and Llama models from 0.8B to 72B, 16\% to 47\% of hallucinated outputs have $\Pmass(t_c; c^*) \geq 0.2$: the model assigned non-trivial mass to the correct concept yet produced a wrong final answer. We call these \emph{commitment failures}, and the rate rises monotonically with scale across both Qwen and Llama families. Commitment failures decompose into two cases: in $\sim$20\%, the greedy first token does not match any surface form of $c^*$ at all (\emph{first-token selection failures}); in the remaining $\sim$80\%, the greedy first token does land on a surface form of $c^*$ but the continuation diverges (\emph{multi-token divergences}). The first sub-population isolates a particularly clean question: when the model put substantial mass on $c^*$ at the commitment step yet selected a token outside $\Scstar$, what does its distribution look like? We compare against \emph{matched correct samples}---correct outputs whose $\Pmass(t_c; c^*) \geq 0.2$, drawn from the same range of correct-concept mass---and find that selection failures have a three-fold lower maximum probability on any single surface-form token of $c^*$ (0.26 vs.\ 0.78). The model has the same \emph{amount} of mass on the correct concept; it just has it spread across alias forms (\texttt{Saint}, \texttt{St}, \texttt{C}) rather than concentrated on one, so a competing concept's single dominant token wins the argmax.

The empirical driver of the scale trend is instruction tuning, not scale itself. The probability assigned to the (wrong) greedy token in first-token selection failures rises monotonically across Instruct models---Qwen: 0.31 (0.8B) to 0.57 (72B); Llama: 0.33 (1B) to 0.49 (70B)---but stays flat at $\sim$0.30 across Base models of the same sizes. The same pattern extends to multi-token answers: in 70B+ Instruct, $H_{t=2}$ is $\approx 0.05$ when the bigram $(y_1, y_2)$ stays on a valid alias prefix of $c^*$ and substantially higher when it diverges (Cohen's $d = 1.29$ across 18 models)---instruction tuning sharpens commitment specifically along $c^*$-aligned phrases. Instruction tuning sharpens commitments at multiple levels, and the same sharpening produces confident correctness when the committed phrase is right and confident misselection when it is wrong---making confident hallucination one face of the broader ``alignment tax'' that has been documented as accuracy loss \citep{ouyang2022training}, calibration loss \citep{openai2023gpt4, hu2025alignment}, and mode collapse. The analytical probe $\Pmass$ (§\ref{sec:setup}), the commitment-failure phenomenon and its scale dependence (§\ref{sec:selection}), the within-population characterization (§\ref{sec:fragmentation}), and the representation-level evidence for instruction-induced sharpening (§\ref{sec:frontload}) together constitute the contribution of this paper.

%=============================================================================
\section{Related Work}
\label{sec:related}
%=============================================================================

\textbf{Confident hallucination and uncertainty-based detection.} Token- and sequence-level uncertainty has been the dominant lens for hallucination detection: perplexity \citep{ren2023outofdistribution}, length-normalized NLL \citep{malinin2021uncertainty}, predictive entropy \citep{kadavath2022language}, importance-weighted scoring \citep[MARS;][]{bakman2024mars}, semantic entropy and its variants \citep{kuhn2023semantic, farquhar2024detecting, ma2025semantic}, sampling consistency \citep[SelfCheckGPT;][]{manakul2023selfcheckgpt}, consistency-confidence aggregation \citep[CoCoA;][]{vashurin2025cocoa}, and confidence elicitation \citep{xiong2024llms}. A growing line of work documents that confident hallucination is itself a phenomenon: \citet{farquhar2024detecting} note that semantic entropy does not address confidently wrong outputs, \citet{simhi2025trust} formalize ``CHOKE'' (certain hallucinations overriding known evidence), and \citet{liu2025delusions} characterize ``delusions'' as high-belief hallucinations. \citet{calderon2026empty} characterize a closely related distinction (``empty shelves'' vs.\ ``lost keys''), finding that recall---not encoding---is the dominant bottleneck even in frontier models, through behavioral fact-level profiling; we provide complementary distributional analysis at the commitment step. These works establish the phenomenon at the response level; we provide its structural account at the commitment step (§\ref{sec:selection},~\ref{sec:fragmentation}).

\textbf{Calibration and the alignment tax.} \citet{kadavath2022language} showed pretrained LLMs are well-calibrated under appropriate elicitation, while \citet{openai2023gpt4} reported that pretraining yields well-calibrated probabilities but RLHF post-training degrades calibration substantially---a finding extended in \citet{xie2024adaptive} for the token-level case and \citet{chhikara2025confidencegap} for the question-type case. More broadly, the ``alignment tax'' originally framed as a drop in task accuracy after RLHF \citep{ouyang2022training} is increasingly understood to include calibration loss and mode collapse: \citet{hu2025alignment} document that alignment makes models overconfident with reduced output diversity, framing this as a calibration--alignment trade-off. We give this its mechanism at the moment of commitment (§\ref{sec:fragmentation}): the same sharpening drives both confident correctness and confident misselection.

\textbf{Token-level decision points.} \citet{wang2025high} show that high-entropy ``forking tokens'' in chain-of-thought reasoning carry a disproportionate share of the learning signal in RLVR. \citet{vassoyan2025critical} identify ``critical tokens'' as decision points where models are most error-prone. We extend the same phenomenon to short-form QA (Figure~\ref{fig:motivation}) and show that the relevant signal at the step is concept-grouped mass, not individual-token entropy.

\textbf{Internal representations and first-token signal.} Truthfulness is linearly decodable from hidden states \citep{burns2023discovering, azaria2023internal, marks2024geometry}; DoLa \citep{chuang2024dola} and ITI \citep{li2024inference} act on this for decoding-time intervention. SEP \citep{kossen2024semantic} introduces token-before-generation (TBG) probing. Token-level analyses have converged on first-token importance \citep{snel2025first, zhao2024first}; HaMI \citep{niu2025hami} adaptively selects informative tokens. We refine this: what matters is not position but the answer-level commitment step (first token in instruction-tuned short-form QA, but migrating in long-form generation; §\ref{sec:eps_align}), and we provide the first systematic Instruct--Base comparison for the TBG setting (§\ref{sec:frontload}).

%=============================================================================
\section{Setup}
\label{sec:setup}
%=============================================================================

\textbf{Concept and $\Pmass$.} Given a query $Q$, an autoregressive LLM $P_\theta$ generates tokens $y_1, y_2, \ldots$ A \emph{concept} $c$ is an equivalence class of token completions denoting the same answer; its first-token surface forms collect into a \emph{concept token set} $\Sc$. We study the \emph{per-step semantic probability mass}
\begin{equation}
\Pmass(t; c) = \sum_{v \in \Sc} P_\theta(v \mid Q, y_{<t}),
\end{equation}
the total mass at step $t$ on any first-token surface form of $c$. To analyze hallucination we set $c = c^*$, the ground-truth concept, with $\Scstar$ built deterministically from the dataset's aliases (Appendix~\ref{app:sc_construction}). Under a latent-concept generation model, $\Pmass(t; c^*)$ approximates the (unobservable) \emph{concept belief} $P_\theta(c^* \mid Q, y_{<t})$ when $\Scstar$ is alias-complete and concepts are well-separated; we make this precise in Appendix~\ref{app:proofs} (Proposition~\ref{prop:pmass}).

\textbf{Models and data.} Our primary scale ablation uses Qwen3.5 \citep{qwen2025qwen3} at four sizes (0.8B, 2B, 4B, 9B) in both Instruct and Base variants, with all 4-bit NF4 quantization, paired with Llama-3.2 (1B, 3B) and Llama-3.1 (8B) \citep{llama2024llama3} in both variants---fourteen models total at small to mid scale. We extend the scale ablation to four large models (Qwen2.5-72B \citep{qwen2024qwen25} and Llama-3.1-70B in both Instruct and Base) for the wrong-token sharpening and commitment-failure rate analyses. We use TriviaQA \citep{joshi2017triviaqa} and NQ-Open \citep{kwiatkowski2019natural} for short-form QA (3{,}000 samples per model) and MMLU \citep{hendrycks2021measuring} together with ARC-Challenge \citep{clark2018think} for multiple-choice QA (2{,}672 samples per model); the representation analyses in §\ref{sec:frontload} use a 1{,}500-sample subset (1{,}000 MCQA + 500 Short-QA).

\textbf{Hallucination.} Throughout, a response is a \emph{hallucination} if it fails substring matching against the ground-truth aliases (Short-QA) or selects a wrong option (MCQA); we use ``hallucinated'' and ``incorrect'' interchangeably, following the convention of SE \citep{farquhar2024detecting} and SEP \citep{kossen2024semantic}.

\textbf{Metrics.} Detection performance is AUROC with hallucination as the positive class. Probes are 5-fold CV logistic regression on hidden states. Calibration uses ECE \citep{guo2017calibration} and Brier score.

%=============================================================================
\section{Results}
\label{sec:results}
%=============================================================================

\subsection{Where in the trajectory does correctness signal live?}
\label{sec:eps_align}

Before turning to the central analysis, we verify the entropy/$\Pmass$ picture in our own data using long-form generation, where the commitment step $t_c$ is not at $t=1$. We collected 500 long-form Qwen3.5-9B Instruct responses (\texttt{Answer in a complete sentence}) and centered each trajectory on its commitment step $t_c$.

Figure~\ref{fig:epoch_aligned}(a) plots entropy as a function of step relative to $t_c$. Entropy peaks sharply at $t_c$ for both correct and hallucinated samples, and hallucinated trajectories carry uniformly higher entropy at every relative step, but the per-sample $H(t_c)$ distributions overlap substantially (Wilcoxon $p = 0.055$). The max-entropy step $t_H$ exactly matches $t_c$ in only 20\% of samples, within one step in 32\%---a noisy localization at best.

Figure~\ref{fig:epoch_aligned}(b) re-plots the same trajectories using $\Pmass(t; c^*)$ instead. $\Pmass$ is essentially zero everywhere except $t_c$, where it spikes to 0.92 for correct samples and 0.77 for hallucinated ones; generated-token probability $P(y_t)$ is nearly flat across both groups (Appendix Figure~\ref{fig:tokprob_aligned}). The relevant fact is not that $\Pmass$ separates the two classes (it has access to ground-truth aliases) but that the gap is small: at the commitment step, hallucinated samples place a substantial 0.77 average mass on the correct concept yet still emit a different one. This sub-population is what we analyze in the rest of the paper.

\textbf{The commitment step is where the information lives.} Figure~\ref{fig:epoch_aligned}(c) makes this concrete: per-step detection AUROC of $\Pmass(t; c^*)$ is at chance two or more steps before $t_c$, climbs to its peak at $t_c$, and decays back to chance within a few steps after. Generated-token probability $P(y_t)$ over the same window is flat (0.55--0.62 throughout, no peak): token-level confidence barely moves around $t_c$, while concept-grouped mass rises and falls sharply. This is the structural justification for studying the model's distribution at $t_c$ specifically rather than aggregating across the trajectory.

\begin{figure}[t]
\centering
\includegraphics[width=\linewidth]{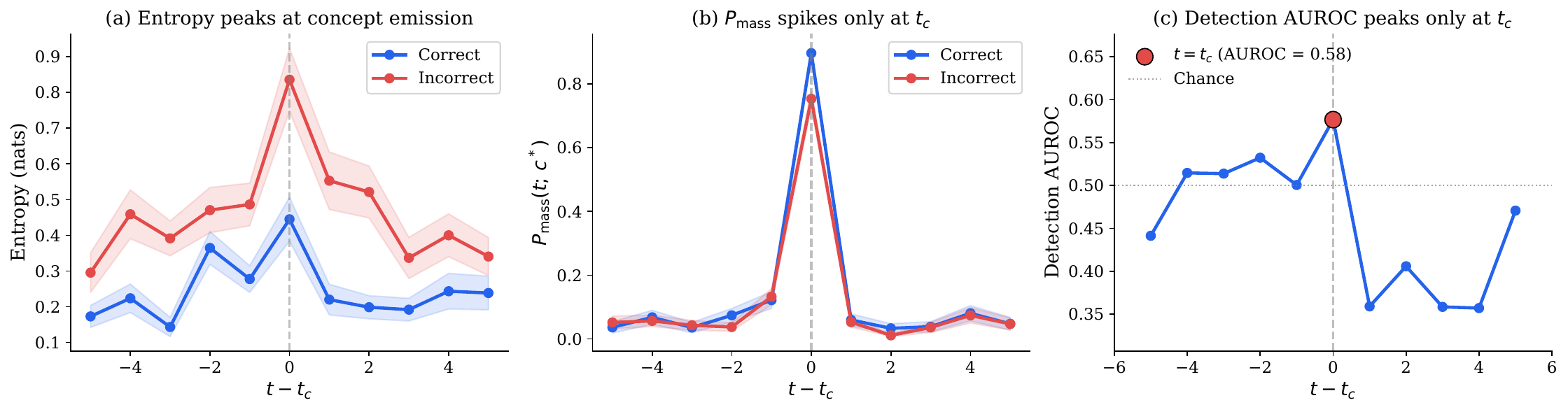}
\caption{500 long-form Qwen3.5-9B Instruct responses aligned to each trajectory's commitment step $t_c$. \emph{(a)~Entropy} peaks at $t_c$ for both groups; hallucinated trajectories run uniformly higher, but per-sample distributions at $t_c$ overlap. \emph{(b)~$\Pmass(t; c^*)$} is essentially zero except at $t_c$, where it concentrates to 0.92 (correct) and 0.77 (hallucinated). The relevant fact is not the separation between groups but the high $\Pmass$ on the hallucinated side: at the commitment step, the model often has substantial mass on the correct concept yet emits a competing one. \emph{(c)~Detection AUROC} of $\Pmass$ peaks sharply at $t_c$ and decays to chance off-step, showing that the predictive information about correctness is localized to the commitment step itself.}
\label{fig:epoch_aligned}
\end{figure}

\subsection{Does the model have the answer when it hallucinates?}
\label{sec:selection}

We now turn to short-form QA, where instruction-tuned models answer immediately and the commitment step is fixed at $t = 1$ (Appendix Table~\ref{tab:firsttoken}). This lets us inspect the model's distribution at a single, known step, and ask the central question: among hallucinated outputs, what does $\Pmass(t_c; c^*)$ look like?

\textbf{$\Pmass$ recovers a coherent confidence signal.} $\Pmass$ is well-calibrated (ECE 0.023--0.096 across 7 Instruct models; Appendix Figure~\ref{fig:calibration_app}), with accuracy increasing monotonically across $\Pmass$ bins, and it assigns substantially more mass to the correct concept than the generated greedy token's probability whenever the answer has multiple surface forms (Appendix Table~\ref{tab:epistemic_app})---it captures the concept-level structure that token-level confidence misses. The per-step result in Figure~\ref{fig:epoch_aligned}(c) confirms $t_c$ is the right inspection point: aggregating $\Pmass$ across the trajectory underperforms the single-step quantity at $t_c$ (Appendix~\ref{app:aggregation}).

\textbf{Commitment failures.} A hallucinated sample is a \emph{commitment failure} if $\Pmass(t_c; c^*) \geq 0.2$: substantial mass on the correct concept yet a wrong final answer. The phenomenon arises because greedy emission depends on the maximum single-token probability, not on $\Pmass$, so individual surface-form tokens of $c^*$ may each be small enough that a competing concept's single dominant token wins, or the multi-token continuation may diverge after a correct first token. CF\% rises monotonically from 16\% at 0.8B to 47\% at 70B Instruct (Table~\ref{tab:cf_by_scale}), holding across both Qwen and Llama families. Larger models do not produce uniformly fewer errors; they shift the error distribution toward commitment failures. The 0.2 threshold is conservative; the trend is robust to threshold choice in $[0.1, 0.4]$ (Appendix~\ref{app:knew_threshold}).

\textbf{Two structural sub-populations.} Commitment failures decompose into two cases at the token level. In a \emph{first-token selection failure}, the greedy emission $y_1$ does not match any surface form of $c^*$ at all (e.g., the answer is \texttt{Saint Petersburg} and the model emits \texttt{Mos} to begin \texttt{Moscow}). In a \emph{multi-token divergence}, $y_1$ does land on a surface form of $c^*$ but the multi-token continuation diverges from any valid alias (e.g., the answer is \texttt{Adam Smith} but the model emits \texttt{Adam Levine}). The two are distributionally different and we treat them separately. Across the scale ablation, first-token selection failures account for roughly 20\% of commitment failures, rising monotonically with scale: from 2.9\% of all hallucinations at 0.8B to 6.0\% at 72B in Qwen Instruct, and from 3.3\% at 1B to 10.2\% at 70B in Llama Instruct (Table~\ref{tab:cf_by_scale}). The remaining $\sim$80\% are multi-token divergences---hallucinations where the first token is on track but the continuation is not. Multi-token divergences are analyzed in detail in §\ref{sec:fragmentation} and Appendix~\ref{app:t2_entropy}.

\begin{table}[t]
\caption{Commitment-failure rate (CF\%) and decomposition into first-token selection failures (SF, greedy $\notin \Scstar$) and multi-token divergences (Div, greedy $\in \Scstar$ but final answer wrong) across the Instruct scale ablation. Acc: Short-QA accuracy. AUROC: $\Pmass(t{=}1)$ AUROC. Halluc: hallucinated sample count. SF\%: SF as fraction of all hallucinations. Full per-model details including Base models in Appendix Table~\ref{tab:scale_full}.}
\label{tab:cf_by_scale}
\centering
\small
\setlength{\tabcolsep}{4pt}
\begin{tabular}{lrrrrrrr}
\toprule
Model & Acc & AUROC & Halluc & CF (CF\%) & SF & Div & SF\% \\
\midrule
Qwen3.5-0.8B Inst   & 8.3\%  & .898 & 2{,}751 & 453 (16\%) & 81  & 372 & 2.9\% \\
Qwen3.5-2B Inst     & 11.0\% & .893 & 2{,}671 & 466 (17\%) & 97  & 369 & 3.6\% \\
Qwen3.5-4B Inst     & 24.6\% & .912 & 2{,}262 & 589 (26\%) & 128 & 461 & 5.7\% \\
Qwen3.5-9B Inst     & 29.4\% & .887 & 2{,}117 & 667 (32\%) & 128 & 539 & 6.0\% \\
Qwen2.5-72B Inst    & 36.4\% & .830 & 1{,}908 & 777 (41\%) & 114 & 663 & 6.0\% \\
\midrule
Llama-3.2-1B Inst   & 14.8\% & .935 & 2{,}559 & 401 (16\%) & 84  & 317 & 3.3\% \\
Llama-3.2-3B Inst   & 31.4\% & .902 & 2{,}054 & 577 (28\%) & 157 & 420 & 7.6\% \\
Llama-3.1-8B Inst   & 32.8\% & .882 & 2{,}018 & 657 (33\%) & 178 & 479 & 8.8\% \\
Llama-3.1-70B Inst  & 44.7\% & .816 & 1{,}659 & \textbf{780 (47\%)} & 170 & 610 & 10.2\% \\
\bottomrule
\end{tabular}
\vspace{-1.0em}
\end{table}

\subsection{Why does the model commit to the wrong token?}
\label{sec:fragmentation}

We focus here on first-token selection failures, the strict sub-population where greedy emission $y_1 \notin \Scstar$ despite $\Pmass(t_c; c^*) \geq 0.2$. These cases are directly inspectable at the token level: the model assigned substantial mass to the correct concept but a single token of a competing concept won the argmax. We ask two questions: (i)~within a model, what distinguishes the distribution at a selection failure from a correct-but-comparable sample (one with a similar amount of mass on $c^*$)? (ii)~How do these distributions change with scale?

\textbf{Within-population: less concentrated correct mass.} To isolate the token-level structure, we compare two groups in Qwen3.5-9B Instruct, both restricted to samples with $\Pmass \geq 0.2$: first-token selection failures (those that hallucinated, $N = 128$) and \emph{matched correct samples} (those that answered correctly, $N = 840$). Both groups have substantial mass on the correct concept; the only difference is the outcome. The maximum probability assigned to any single surface form of $c^*$, $\max_{v \in \Scstar} P_\theta(v)$, is dramatically smaller in the failure group: mean 0.26 vs.\ 0.78 (Welch $t = 47.6$, $p < 10^{-180}$, Cohen's $d = 2.98$; Appendix~\ref{app:stats_primer}). The same comparison across all 18 models gives $d < 0$ in 100\% of cases (median $|d| = 1.93$, range $[1.01, 4.30]$); within Instruct models $|d|$ tends to grow with scale (Qwen Inst: 1.34$\to$2.98$\to$4.30 from 0.8B to 72B; Appendix Table~\ref{tab:within_population}). Vocabulary fragmentation across alias forms (Figure~\ref{fig:fragmentation}, e.g.\ \texttt{Saint}, \texttt{St}, \texttt{C}) is the most concrete realization in small and Base models; in mid-to-large Instruct models, the within-concept distribution has typically already collapsed onto a single alias and the SF--Corr gap arises from a different route, characterized below.

\begin{figure}[t]
\centering
\includegraphics[width=\linewidth]{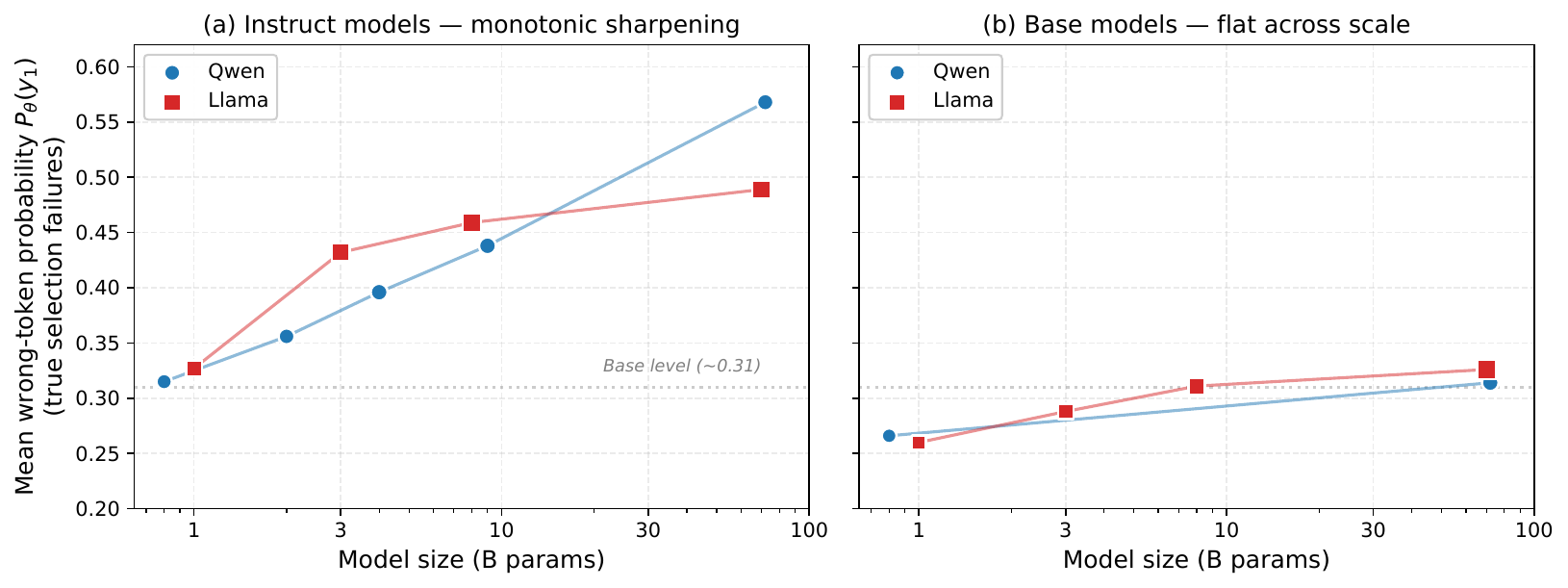}
\caption{Within first-token selection failures, mean wrong-token probability $P_\theta(y_1)$ across models. \emph{(a) Instruct models}: monotonic sharpening from $\sim$0.31 at 1B to $\sim$0.49--0.57 at 70B+, in both families. \emph{(b) Base models}: flat at $\sim$0.26--0.33 across the same scale range. Marker size $\propto$ number of true selection failures; the dotted line at 0.31 marks the typical Base level. The contrast is family-independent: instruction tuning---not scale alone---is the driver.}
\label{fig:two_patterns}
\vspace{-2.0em}
\end{figure}

\textbf{Across scale: monotonic sharpening, modulated by instruction tuning.} Within first-token selection failures, the greedy emitted token is by definition outside $\Scstar$; we call its probability $P_\theta(y_1)$ the \emph{wrong-token probability}. This rises monotonically with model size in instruction-tuned models: Qwen Instruct 0.31 (0.8B) $\to$ 0.36 $\to$ 0.40 $\to$ 0.44 $\to$ 0.57 (72B); Llama Instruct 0.33 (1B) $\to$ 0.43 $\to$ 0.46 $\to$ 0.49 (70B) (Figure~\ref{fig:two_patterns}, left). Base models behave differently: across the same size range, wrong-token probability stays near 0.30 (Llama Base: 0.26 at 1B to 0.33 at 70B; Qwen Base 0.8B/72B: 0.27/0.31). Scale alone does not produce sharpening; the combination of scale and instruction tuning does. The same sharpening that produces front-loaded correctness signal in §\ref{sec:frontload} produces decisive misselection when the committed concept is wrong.

\textbf{Sharpening extends to multi-token answers.} The same sharpening operates beyond the first token. Within multi-token divergences, $H_{t=2}$ (entropy of the next-token distribution after $y_1$) measures whether the model has committed to a specific multi-token phrase at $t=1$ (low $H_{t=2}$) or is still deciding. We classify each divergence by whether the second token continues an alias of $c^*$. \emph{Type A}: bigram $(y_1, y_2)$ matches the start of a valid alias but the continuation diverges into a different entity (e.g., \texttt{George Washington Carver}---the agricultural scientist---when the answer is \texttt{George Washington} the first U.S.\ president; the bigram \texttt{George Washington} is shared with the alias prefix, but \texttt{Carver} fixes a different person). \emph{Type B}: bigram diverges already at $y_2$ (e.g., \texttt{Adam Lambert} when the answer is \texttt{Adam Smith}: $y_1 = $ \texttt{Adam} is in $\Scstar$ but \texttt{Lambert} breaks the alignment). Across 18 Instruct/Base models, Type A divergences have substantially lower $H_{t=2}$ than Type B (median Cohen's $d = 1.29$, $d > 1$ in 100\% of models with sufficient $N$; Appendix~\ref{app:t2_entropy})---when the bigram is on track, the continuation is near-deterministic. The Type A fraction grows with both scale and instruction tuning, from 21--44\% at 0.8B/1B to 77--83\% at 70B+ Instruct. At 70B+ Instruct, $H_{t=2}$ within Type A divergences is 0.05--0.10---the model commits to wrong multi-token phrases with residual entropy comparable to deterministic continuations.

\textbf{Two faces of commitment failure.} The within-population analysis above measures top1 alias mass without normalizing by $\Pmass(c^*)$. To separate \emph{within-concept} structure (how the correct mass is distributed across alias forms) from \emph{between-concept} structure (how the wrong greedy token compares to the correct concept), we measure two ratios on SF samples across all 18 models: $D_2 = \max_{v \in \Scstar} P(v) / \Pmass(c^*)$ and $D_3 = P(\text{greedy}) / \Pmass(c^*)$. Within Llama Instruct, $D_2$ grows monotonically (1B 0.76 $\to$ 70B 0.99) while Llama Base plateaus (0.66 $\to$ 0.81); Qwen2.5-72B shows the same contrast (Inst 1.00 vs Base 0.76). $D_3$ follows the same pattern: Inst grows monotonically with scale (Llama 1.13$\to$1.45; Qwen 1.13$\to$1.65), Base stays flat ($\sim$1.0--1.13). The two ratios separate two failure modes: \emph{fragmentation-driven failures} (low $D_2$, low $D_3$) in small or Base models where the correct mass is split across alias tokens (the regime Figure~\ref{fig:fragmentation} illustrates), and \emph{wrong-attractor failures} (high $D_2$, high $D_3$) dominating large Instruct models, where the correct concept has collapsed onto a single alias but a wrong concept's token is even sharper. Both arise from the same instruction-induced sharpening: weak sharpening leaves correct mass spread; strong sharpening collapses fragmentation but strengthens wrong attractors at least as fast (Appendix~\ref{app:d2_d3}).

Instruction-induced sharpening therefore acts at three structural levels: at the first token (selection failures, sharper wrong-token probability), across multi-token answers (early commitment to specific phrase continuations), and within the correct concept's alias distribution ($D_2$ collapse).

\subsection{When does the model ``know'' it is going to fail?}
\label{sec:frontload}

A separate but related observation is what makes $t_c = 1$ in instruction-tuned short-form QA. The prompt format alone is not enough: Base models, given the same short-form prompt, emit filler tokens first and delay $t_c$ to later steps (Appendix Table~\ref{tab:firsttoken}). The Instruct--Base contrast lets us ask whether the front-loading is purely an output-formatting effect or reflects deeper changes in the representation.

\textbf{Output-level detection.} On MCQA, Instruct models reach near-perfect detection AUROC (0.974--0.999) using only $P(\text{correct option})$, while Base models span 0.558--0.748 (Figure~\ref{fig:probe_attn_v2}, right; +0.29 average gap).

\textbf{Attention to the question.} At $t = 1$, Instruct models allocate a higher fraction of last-layer attention to question tokens (+0.09 average; Figure~\ref{fig:probe_attn_v2}, middle), consistent with retrieving the answer concept from $Q$ rather than first emitting filler.

\textbf{Hidden-state probes (pre-generation).} Logistic regression on the last-layer hidden state at $t = 1$, before any token is generated, yields Instruct $>$ Base in all four sizes (+0.08 average; Figure~\ref{fig:probe_attn_v2}, left). The pattern holds across nearly all layers (Appendix~\ref{app:multilayer}), peaking at mid-layers, ruling out a purely output-formatting explanation. This is the first systematic Instruct--Base comparison for the TBG setting of \citet{kossen2024semantic}: pre-generation probe AUROC is 0.61--0.87 for Instruct versus 0.50--0.63 for Llama Base, and Base models gain substantially more from one generated token (avg pre$\to$post $\Delta_B = +0.030$ vs.\ $\Delta_I = +0.005$; Appendix Table~\ref{tab:postgen_app}). Correctness-predictive information is genuinely front-loaded in Instruct models.

\begin{figure}[t]
\centering
\includegraphics[width=\linewidth]{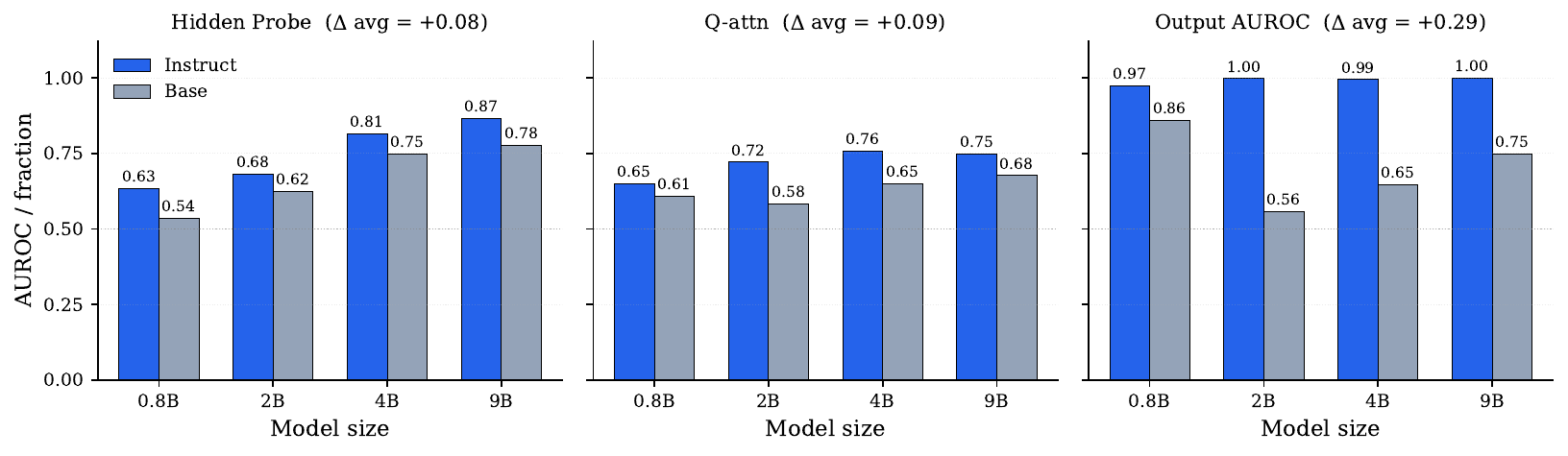}
\caption{Three-level Instruct--Base comparison at $t = 1$. Hidden Probe: 5-fold CV AUROC on last-layer hidden states (MCQA, $N{=}1{,}000$). Q-attn: fraction of last-layer attention on question tokens (Short-QA, $N{=}500$). Output AUROC: $P(\text{correct option})$ on MCQA. Average Instruct--Base gaps: $+0.08$ (Hidden Probe), $+0.09$ (Q-attn), $+0.29$ (Output AUROC).}
\label{fig:probe_attn_v2}
\vspace{-1.0em}
\end{figure}

The output-level gap (+0.29) is larger than the hidden-state gap (+0.08): instruction tuning's effect is partly representational (information is in the hidden states) and substantially in the output mapping (the projection amplifies it sharply).

%=============================================================================
\section{Discussion and Limitations}
\label{sec:discussion}
%=============================================================================

\textbf{Hallucination as commitment failure.} Across 18 models from 0.8B to 72B, 16\% to 47\% of Instruct hallucinations leave non-trivial mass on the correct concept at the commitment step yet produce a wrong final answer, with the rate rising monotonically with scale. As models scale, more hallucinations come from commitment failures despite the population-level distribution including the correct answer, not from the answer being absent. The within-population finding sharpens this: at matched $\Pmass$, failures consistently have lower top-token mass on $c^*$ ($d < 0$ in 100\% of 18 models, $|d|$ growing with scale within each Instruct family from 1.34 to 4.30; Appendix~\ref{app:within_population_app}). The structural difference between hallucination and correctness at comparable concept-level mass is whether any single surface form is concentrated enough to win. The empirical driver is instruction-induced sharpening: Instruct models sharpen first-token commitments with scale (0.31 to 0.57), Base models remain flat at $\sim$0.30. The same sharpening produces front-loaded correctness signal (§\ref{sec:frontload}) and decisive misselection when the committed concept is wrong.

\textbf{Sharpening operates at multiple granularities.} The first-token effect extends both inward and forward. Within multi-token divergences, the second-token entropy after an alias-prefix-aligned bigram falls steeply with scale and instruction tuning, reaching 0.05--0.10 at 70B+ Instruct (Appendix~\ref{app:t2_entropy}); by the second token, the model has effectively committed to a specific multi-token continuation, and a wrong continuation is selected with the same residual entropy as a deterministic one. The same sharpening also operates within the correct concept's alias distribution: in 70B+ Instruct, mass on the correct concept has typically collapsed onto a single alias token (Appendix~\ref{app:d2_d3}), removing fragmentation as a recoverable failure mode. Confident hallucination is therefore not a momentary slip at $t=1$ but the natural endpoint of a sharpening process operating at three structural levels---first-token selection, multi-token phrase commitment, and within-concept alias collapse---more decisive answers when right, more decisive misselections when wrong.

\textbf{Implications for the picture of confident hallucination.} A confident hallucination is one where the model places high probability on a wrong answer's tokens, and the standard framing treats this as evidence that the model has the wrong answer in its distribution and not the right one. Our results complicate this picture: in commitment failures the model has placed substantial mass on the \emph{correct} concept ($\Pmass \geq 0.2$) while still emitting a wrong token confidently. Token-level confidence is not coming from the absence of the right concept but from concentration on the wrong concept in spite of it. The same sharpness produces confident correctness when the committed concept is right. ``Confidently wrong'' and ``confidently right'' are two outcomes of one distributional disposition, not two different epistemic states---which may explain why uncertainty does not flag confident hallucinations \citep{simhi2025trust, liu2025delusions, farquhar2024detecting}.

\textbf{Limitations.} $\Pmass$ is an analytical probe, not a deployable detector: it requires the ground-truth alias set $\Scstar$ as input, so any practical use depends on inferring $\Scstar$ from context (e.g., by clustering top-$k$ tokens at $t_c$ by semantic equivalence). Within multi-token divergences, $\Pmass(t=1; c^*) \geq 0.2$ does not distinguish concept-level belief from phrase-level commitment to specific multi-token continuations (Appendix~\ref{app:t2_entropy}); the within-population analysis in §\ref{sec:fragmentation} controls for this by restricting to first-token selection failures. Our experiments use 1--3 token answers and greedy decoding; concept-segmented $\Pmass$ for longer generation, and the behavior of commitment failures under temperature sampling or top-$p$ truncation, are natural extensions. The analyses also span only Qwen and Llama families; we expect the same instruction-induced sharpening pattern in other open-weight families, but cannot verify it for closed models without first-token distribution access.

\textbf{Future work.} The commitment-failure phenomenon directly suggests two directions. First, greedy decoding has a structural limitation when $\Pmass(c^*)$ is high but split across surface forms: a \emph{concept-aware} decoding rule that argmaxes over alias-clustered top-$k$ tokens---rather than over individual vocabulary entries---would convert a meaningful fraction of selection failures into correct answers without retraining. Quantifying the upper bound and approximating it with semantic-similarity clustering of top-$k$ is a concrete next step. Second, our setup fixes the commitment step at $t=1$ via short-form prompting; long-form generation has multiple commitment events---domain commitment (\texttt{Britain}), answer commitment (\texttt{Nicola}), possibly others (rhetorical-frame, sub-claim)---that a richer typology should distinguish. Identifying these reliably is itself a problem: entropy is a noisy localizer (exact match in only 20\% of long-form trajectories, §\ref{sec:eps_align}), so non-entropy signals---hidden-state probes, attention concentration on $Q$, or $\Pmass$ rate of change $\Delta\Pmass(t)$ at each candidate spike---are likely better candidates and worth systematic comparison.

%=============================================================================
\section{Conclusion}
%=============================================================================

We asked what is happening at the moment of hallucination, viewed through the model's distribution at the commitment step. Defining concepts as equivalence classes of token completions and introducing per-step semantic probability mass as an analytical probe, we found that a substantial fraction of Instruct hallucinations are commitment failures: the model puts non-trivial mass on the correct concept yet produces a wrong final answer, with the rate rising monotonically with scale. Larger models do not just know more; they also misfire more often on what they know. Within these failures, the structural difference from matched correct generations is not whether the correct concept is represented, but how its mass is distributed across surface forms. Across scale, the same sharpening pattern operates at the first token, across multi-token continuations, and within the correct concept's alias distribution---uniformly in Instruct models, absent in Base. This reframes confident hallucination as a structural consequence of how mass is shaped at the commitment step, and situates it as one face of the broader alignment tax---suggesting concept-aware decoding and finer commitment-step typologies as natural follow-ups.

\clearpage
\bibliographystyle{plainnat}
\bibliography{refs}

@inproceedings{bakman2024mars,
  title={MARS: Meaning-aware response scoring for uncertainty estimation in generative LLMs},
  author={Bakman, Y. and others},
  booktitle={ACL},
  year={2024}
}

@article{clark2018think,
  title={Think you have solved question answering? Try ARC},
  author={Clark, P. and others},
  journal={arXiv preprint arXiv:1803.05457},
  year={2018}
}

@article{farquhar2024detecting,
  title={Detecting hallucinations in large language models using semantic entropy},
  author={Farquhar, S. and Kossen, J. and Kuhn, L. and Gal, Y.},
  journal={Nature},
  volume={630},
  pages={625--630},
  year={2024}
}

@inproceedings{hendrycks2021measuring,
  title={Measuring massive multitask language understanding},
  author={Hendrycks, D. and others},
  booktitle={ICLR},
  year={2021}
}

@article{ji2023survey,
  title={Survey of hallucination in natural language generation},
  author={Ji, Z. and others},
  journal={ACM Computing Surveys},
  volume={55},
  number={12},
  pages={1--38},
  year={2023}
}

@inproceedings{joshi2017triviaqa,
  title={TriviaQA: A large scale distantly supervised challenge dataset for reading comprehension},
  author={Joshi, M. and Choi, E. and Weld, D. S. and Zettlemoyer, L.},
  booktitle={ACL},
  year={2017}
}

@article{kadavath2022language,
  title={Language models (mostly) know what they know},
  author={Kadavath, S. and others},
  journal={arXiv preprint arXiv:2207.05221},
  year={2022}
}

@inproceedings{kossen2024semantic,
  title={Semantic entropy probes: Robust and cheap hallucination detection in LLMs},
  author={Kossen, J. and others},
  booktitle={NeurIPS},
  year={2024}
}

@inproceedings{li2024inference,
  title={Inference-time intervention: Eliciting truthful answers from a language model},
  author={Li, K. and others},
  booktitle={NeurIPS},
  year={2024}
}

@article{llama2024llama3,
  title={The Llama 3 herd of models},
  author={{Llama Team}},
  journal={arXiv preprint arXiv:2407.21783},
  year={2024}
}

@article{kwiatkowski2019natural,
  title={Natural questions: A benchmark for question answering research},
  author={Kwiatkowski, T. and others},
  journal={TACL},
  volume={7},
  pages={453--466},
  year={2019}
}

@article{ma2025semantic,
  title={Semantic energy: Detecting LLM hallucination beyond entropy},
  author={Ma, H. and others},
  journal={arXiv preprint arXiv:2508.14496},
  year={2025}
}

@inproceedings{malinin2021uncertainty,
  title={Uncertainty estimation in autoregressive structured prediction},
  author={Malinin, A. and Gales, M.},
  booktitle={ICLR},
  year={2021}
}

@article{niu2025hami,
  title={Robust hallucination detection in LLMs via adaptive token selection},
  author={Niu, X. and others},
  journal={arXiv preprint arXiv:2504.07863},
  year={2025}
}

@article{qwen2025qwen3,
  title={Qwen3 technical report},
  author={{Qwen Team}},
  journal={arXiv preprint arXiv:2505.09388},
  year={2025}
}

@inproceedings{ren2023outofdistribution,
  title={Out-of-distribution detection and selective generation for conditional language models},
  author={Ren, J. and others},
  booktitle={ICLR},
  year={2023}
}

@article{snel2025first,
  title={First hallucination tokens are different from conditional ones},
  author={Snel, J. and Oh, S. J.},
  journal={arXiv preprint arXiv:2507.20836},
  year={2025}
}

@article{zhao2024first,
  title={The first to know: How token distributions reveal hidden knowledge in large vision-language models},
  author={Zhao, Q. and Xu, M. and Gupta, K. and Asthana, A. and Zheng, L. and Gould, S.},
  journal={arXiv preprint arXiv:2403.09037},
  year={2024}
}

@inproceedings{guo2017calibration,
  title={On calibration of modern neural networks},
  author={Guo, C. and Pleiss, G. and Sun, Y. and Weinberger, K. Q.},
  booktitle={ICML},
  year={2017}
}

@inproceedings{xiong2024llms,
  title={Can {LLM}s express their uncertainty? An empirical evaluation of confidence elicitation in {LLM}s},
  author={Xiong, M. and Hu, Z. and Lu, X. and Li, Y. and Fu, J. and He, J. and Hooi, B.},
  booktitle={ICLR},
  year={2024}
}

@inproceedings{chuang2024dola,
  title={Do{L}a: Decoding by contrasting layers improves factuality in large language models},
  author={Chuang, Y.-S. and Xie, Y. and Luo, H. and Kim, Y. and Glass, J. and He, P.},
  booktitle={ICLR},
  year={2024}
}

@inproceedings{burns2023discovering,
  title={Discovering latent knowledge in language models without supervision},
  author={Burns, C. and Ye, H. and Klein, D. and Steinhardt, J.},
  booktitle={ICLR},
  year={2023}
}

@inproceedings{azaria2023internal,
  title={The internal state of an {LLM} knows when it's lying},
  author={Azaria, A. and Mitchell, T.},
  booktitle={Findings of EMNLP},
  year={2023}
}

@inproceedings{marks2024geometry,
  title={The geometry of truth: Emergent linear structure in large language model representations of true/false datasets},
  author={Marks, S. and Tegmark, M.},
  booktitle={COLM},
  year={2024}
}

@inproceedings{manakul2023selfcheckgpt,
  title={SelfCheckGPT: Zero-resource black-box hallucination detection for generative large language models},
  author={Manakul, P. and Liusie, A. and Gales, M.},
  booktitle={EMNLP},
  year={2023}
}

@inproceedings{kuhn2023semantic,
  title={Semantic uncertainty: Linguistic invariances for uncertainty estimation in natural language generation},
  author={Kuhn, L. and Gal, Y. and Farquhar, S.},
  booktitle={ICLR},
  year={2023}
}

@inproceedings{vashurin2025cocoa,
  title={CoCoA: A minimum Bayes risk framework bridging confidence and consistency for uncertainty quantification in LLMs},
  author={Vashurin, R. and Goloburda, M. and Ilina, A. and Rubashevskii, A. and Nakov, P. and Shelmanov, A. and Panov, M.},
  booktitle={NeurIPS},
  year={2025}
}

@inproceedings{wang2025high,
  author    = {Wang, Shenzhi and Yu, Le and Gao, Chang and Zheng, Chujie and Liu, Shixuan and Lu, Rui and Dang, Kai and Chen, Xiong-Hui and Yang, Jianxin and Zhang, Zhenru and Liu, Yuqiong and Yang, An and Zhao, Andrew and Yue, Yang and Song, Shiji and Yu, Bowen and Huang, Gao and Lin, Junyang},
  title     = {Beyond the 80/20 Rule: High-Entropy Minority Tokens Drive Effective Reinforcement Learning for {LLM} Reasoning},
  booktitle = {Advances in Neural Information Processing Systems (NeurIPS)},
  year      = {2025},
  note      = {arXiv:2506.01939}
}

@article{vassoyan2025critical,
  author    = {Vassoyan, Jean and Beau, Nathana{\"e}l and Plaud, Roman},
  title     = {Ignore the {KL} Penalty! Boosting Exploration on Critical Tokens to Enhance {RL} Fine-Tuning},
  journal   = {Findings of the North American Chapter of the Association for Computational Linguistics (NAACL)},
  year      = {2025},
  note      = {arXiv:2502.06533}
}

@inproceedings{simhi2025trust,
  author    = {Simhi, Adi and Itzhak, Itay and Barez, Fazl and Stanovsky, Gabriel and Belinkov, Yonatan},
  title     = {Trust Me, {I}'m Wrong: {LLM}s Hallucinate with Certainty Despite Knowing the Answer},
  booktitle = {Findings of the Association for Computational Linguistics: EMNLP 2025},
  pages     = {14665--14688},
  year      = {2025},
  note      = {arXiv:2502.12964}
}

@article{liu2025delusions,
  author    = {Xu, Hongshen and Yang, Zixv and Zhu, Zichen and Lan, Kunyao and Wang, Zihan and Wu, Mengyue and Ji, Ziwei and Chen, Lu and Fung, Pascale and Yu, Kai},
  title     = {Delusions of Large Language Models},
  journal   = {arXiv preprint arXiv:2503.06709},
  year      = {2025}
}

@article{calderon2026empty,
  author    = {Calderon, Nitay and Ben-David, Eyal and Gekhman, Zorik and Ofek, Eran and Yona, Gal},
  title     = {Empty Shelves or Lost Keys? {R}ecall Is the Bottleneck for Parametric Factuality},
  journal   = {arXiv preprint arXiv:2602.14080},
  year      = {2026}
}

@article{qwen2024qwen25,
  author    = {Yang, An and Yang, Baosong and Zhang, Beichen and Hui, Binyuan and Zheng, Bo and Yu, Bowen and Li, Chengyuan and Liu, Dayiheng and Huang, Fei and Wei, Haoran and Lin, Huan and Yang, Jian and Tu, Jianhong and Zhang, Jianwei and Yang, Jianxin and Yang, Jiaxi and Zhou, Jingren and Lin, Junyang and Dang, Kai and Lu, Keming and others},
  title     = {{Qwen2.5} Technical Report},
  journal   = {arXiv preprint arXiv:2412.15115},
  year      = {2024}
}

@article{openai2023gpt4,
  author    = {{OpenAI}},
  title     = {{GPT-4} Technical Report},
  journal   = {arXiv preprint arXiv:2303.08774},
  year      = {2023}
}

@inproceedings{xie2024adaptive,
  author    = {Xie, Johnathan and Chen, Annie S. and Lee, Yoonho and Mitchell, Eric and Finn, Chelsea},
  title     = {Calibrating Language Models with Adaptive Temperature Scaling},
  booktitle = {Proceedings of the 2024 Conference on Empirical Methods in Natural Language Processing (EMNLP)},
  year      = {2024},
  note      = {arXiv:2409.19817}
}

@article{chhikara2025confidencegap,
  author    = {Chhikara, Prateek},
  title     = {Mind the Confidence Gap: Overconfidence, Calibration, and Distractor Effects in Large Language Models},
  journal   = {Transactions on Machine Learning Research (TMLR)},
  year      = {2025},
  note      = {arXiv:2502.11028}
}

@inproceedings{ouyang2022training,
  author    = {Ouyang, Long and Wu, Jeffrey and Jiang, Xu and Almeida, Diogo and Wainwright, Carroll L. and Mishkin, Pamela and Zhang, Chong and Agarwal, Sandhini and Slama, Katarina and Ray, Alex and others},
  title     = {Training Language Models to Follow Instructions with Human Feedback},
  booktitle = {Advances in Neural Information Processing Systems (NeurIPS)},
  year      = {2022}
}

@article{hu2025alignment,
  author    = {Hu, Tiancheng and Minixhofer, Benjamin and Collier, Nigel},
  title     = {Navigating the Alignment-Calibration Trade-off: A {P}areto-Superior Frontier via Model Merging},
  journal   = {arXiv preprint arXiv:2510.17426},
  year      = {2025}
}

@book{cohen1988statistical,
  author    = {Cohen, Jacob},
  title     = {Statistical Power Analysis for the Behavioral Sciences},
  edition   = {2},
  publisher = {Lawrence Erlbaum Associates},
  year      = {1988}
}

@article{welch1947generalization,
  author    = {Welch, B. L.},
  title     = {The generalization of `{S}tudent's' problem when several different population variances are involved},
  journal   = {Biometrika},
  volume    = {34},
  number    = {1/2},
  pages     = {28--35},
  year      = {1947}
}

@article{mann1947test,
  author    = {Mann, Henry B. and Whitney, Donald R.},
  title     = {On a test of whether one of two random variables is stochastically larger than the other},
  journal   = {The Annals of Mathematical Statistics},
  volume    = {18},
  number    = {1},
  pages     = {50--60},
  year      = {1947}
}

@book{fisher1925statistical,
  author    = {Fisher, Ronald A.},
  title     = {Statistical Methods for Research Workers},
  publisher = {Oliver and Boyd},
  year      = {1925}
}

\clearpage
%=============================================================================
\appendix
%=============================================================================

\clearpage
\section{Theoretical Analysis}
\label{app:proofs}

We work under a latent-concept generation model: at each step the model implicitly considers candidate concepts $c \in \mathcal{C}$ before emitting tokens, so $P_\theta(y_t \mid Q, y_{<t}) = \sum_{c \in \mathcal{C}} P_\theta(y_t \mid c, Q, y_{<t}) \cdot P_\theta(c \mid Q, y_{<t})$. We refer to $P_\theta(c \mid Q, y_{<t})$ as the model's \emph{concept belief} and to $P_\theta(y_t \mid c, Q, y_{<t})$ as the concept-conditioned emission distribution.

\begin{proposition}[$\Pmass$ as Concept-Belief Proxy]
\label{prop:pmass}
Let $\gamma_c = P_\theta(y_t \in \Sc \mid c, Q, y_{<t})$ (\emph{completeness} of $\Sc$ under $c$) and $\epsilon = \max_{c' \neq c} P_\theta(y_t \in \Sc \mid c', Q, y_{<t})$ (\emph{leakage} from competing concepts). With $K = |\mathcal{C}|$,
\begin{equation*}
\bigl| \Pmass(t; c) - P_\theta(c \mid Q, y_{<t}) \bigr| \;\leq\; (1-\gamma_c) + (K-1)\epsilon.
\end{equation*}
\end{proposition}

\begin{proof}[Proof of Proposition~\ref{prop:pmass}]
By the law of total probability,
\begin{equation}
    \Pmass(t) = \sum_{v \in \Sc} \sum_{c \in \mathcal{C}} P_\theta(v \mid c, Q, y_{<t}) \cdot P_\theta(c \mid Q, y_{<t}) = \sum_{c \in \mathcal{C}} P_\theta(c \mid Q, y_{<t}) \cdot \alpha_c(t),
\end{equation}
where $\alpha_c(t) = P_\theta(y_t \in \Sc \mid c, Q, y_{<t}) \in [\gamma_c, 1]$ for the target concept and $\alpha_{c'}(t) \in [0, \epsilon]$ for $c' \neq c$. Then
\begin{align}
    \Pmass(t) &\leq P_\theta(c \mid Q, y_{<t}) + (1 - P_\theta(c \mid Q, y_{<t})) \cdot \epsilon, \\
    \Pmass(t) &\geq P_\theta(c \mid Q, y_{<t}) \cdot \gamma_c.
\end{align}
Combining gives $|\Pmass(t) - P_\theta(c \mid Q, y_{<t})| \leq (1-\gamma_c) + (K-1)\epsilon$.
\end{proof}

\begin{proposition}[Posterior Concentration at Concept Emission, Auxiliary]
\label{prop:concentration}
Let $t_c$ denote the first step at which a token from some concept's first-token set is emitted, $y_{t_c} \in S_{\hat{c}}$. Under the latent-concept model with completeness $\gamma_{\hat{c}}$ and leakage $\epsilon$,
\begin{equation}
    P_\theta(\hat{c} \mid Q, y_{\leq t_c}) \;\geq\; \frac{\gamma_{\hat{c}} \cdot P_\theta(\hat{c} \mid Q, y_{<t_c})}{\gamma_{\hat{c}} \cdot P_\theta(\hat{c} \mid Q, y_{<t_c}) + (K-1)\epsilon}.
\end{equation}
With $\epsilon = 0$, the posterior concentrates to a delta function on $\hat{c}$, regardless of correctness. This formalizes the deterministic continuation between entropy spikes in Figure~\ref{fig:motivation}; it is auxiliary to the empirical analysis in §\ref{sec:results}.
\end{proposition}

\begin{proof}
Bayes' rule at step $t_c$ gives
$P_\theta(\hat{c} \mid Q, y_{\leq t_c}) = P_\theta(y_{t_c} \mid \hat{c}, Q, y_{<t_c}) \cdot P_\theta(\hat{c} \mid Q, y_{<t_c}) / P_\theta(y_{t_c} \mid Q, y_{<t_c})$. The numerator is at least $\gamma_{\hat{c}} \cdot P_\theta(\hat{c} \mid Q, y_{<t_c})$ by completeness. The denominator decomposes via total probability and is bounded above by $P_\theta(y_{t_c} \mid \hat{c}, Q, y_{<t_c}) \cdot P_\theta(\hat{c} \mid Q, y_{<t_c}) + (K-1)\epsilon$, using $P_\theta(y_{t_c} \mid c', Q, y_{<t_c}) \leq \epsilon$ for $c' \neq \hat{c}$. With $P_\theta(y_{t_c} \mid \hat{c}, Q, y_{<t_c}) \leq 1$ in the denominator, the bound follows.
\end{proof}

\clearpage
\section{Statistical Primer}
\label{app:stats_primer}

We use a small number of standard statistics throughout this paper; this appendix is a brief reference for readers unfamiliar with them.

\textbf{Cohen's $d$} (effect size, \citealp{cohen1988statistical}). For two groups with means $\mu_1, \mu_2$ and pooled standard deviation $\sigma$, $d = (\mu_1 - \mu_2) / \sigma$. It expresses how far apart the two group means are in units of typical within-group spread. Conventional thresholds: $|d| < 0.2$ small, $|d| \approx 0.5$ medium, $|d| > 0.8$ large, $|d| > 1.5$ very large. Effect-size statistics like $d$ complement $p$-values, which only tell you whether a difference is reliably nonzero, not how big it is.

\textbf{Welch's $t$-test} \citep{welch1947generalization}. A two-sample $t$-test that does not assume equal variances between the two groups; reports a $t$-statistic and a $p$-value for the null hypothesis ``the two group means are equal.''

\textbf{Mann--Whitney $U$ test} \citep{mann1947test}. A non-parametric two-sample test on whether one group's values systematically dominate the other's. It does not assume normality. We use it to confirm Welch's $t$-test results in cases where group distributions are skewed.

\textbf{Pearson's $r$}. Linear correlation coefficient between two scalar quantities, in $[-1, 1]$.

\textbf{Fisher's combined $p$} \citep{fisher1925statistical}. A way to combine $k$ independent $p$-values into a single one: $-2 \sum_i \ln p_i$ is $\chi^2$-distributed with $2k$ degrees of freedom under a global null. Used here for pooled meta-analysis across models.

\textbf{Cohen's $d$ sign convention}. Throughout, when we report negative $d$ between SF and Correct groups, the sign indicates SF $<$ Corr (i.e., the failure group has lower top-token mass on $c^*$); the magnitude is what matters for the conclusion.

\clearpage
\section{Spread-Based Diagnostics}
\label{app:fragmentation_diagnostics}

The within-population result reported in §\ref{sec:fragmentation} uses the simple statistic $\max_{v \in \Scstar} P_\theta(v)$. We tested an alternative summary statistic, the inverse Simpson diversity index (also called effective number of types):

\begin{equation*}
\Spread(c^*; i, t) \;=\; \frac{\bigl(\Pmass(t; c^*)\bigr)^2}{\sum_{v \in \Scstar} P_\theta(v \mid Q_i, y_{<t})^2}.
\end{equation*}

$\Spread = 1$ when all mass is on a single token; $\Spread = k$ when uniformly distributed over $k$ tokens. The two summary statistics measure different things---$\max$ measures the dominant token's probability, $\Spread$ measures the evenness of the distribution---and within first-token selection failures, the two give different pictures.

\textbf{Across-scale comparison (Qwen and Llama Instruct).} The table below reports $\Spread(c^*)$ \emph{conditioned on the sample being a first-token selection failure}---i.e., the average over the SF subset, not over all samples. This is the relevant statistic for asking how dispersed the correct concept's mass is when a selection failure occurs.

\begin{center}
\small
\begin{tabular}{lcc}
\toprule
Model & Avg $\Spread(c^*) \mid \mathrm{SF}$ & CF\% \\
\midrule
Qwen3.5 0.8B Inst & 1.91 & 16.5\% \\
Qwen3.5 2B Inst   & 1.43 & 17.4\% \\
Qwen3.5 4B Inst   & 1.53 & 26.0\% \\
Qwen3.5 9B Inst   & 1.39 & 31.5\% \\
Qwen2.5 72B Inst  & 1.11 & 40.7\% \\
\midrule
Llama-3.2 1B Inst  & 1.83 & 14.4\% \\
Llama-3.2 3B Inst  & 1.30 & 28.1\% \\
Llama-3.1 8B Inst  & 1.34 & 32.7\% \\
Llama-3.1 70B Inst & 1.18 & 47.0\% \\
\bottomrule
\end{tabular}
\end{center}

The SF-conditional $\Spread$ decreases with scale in both families: Qwen3.5--Qwen2.5 from 1.91 (0.8B) to 1.11 (72B), Llama-3.2--Llama-3.1 from 1.83 (1B) to 1.18 (70B). At first glance this might seem paradoxical: $\Spread$ measures how scattered the correct concept's mass is across alias tokens, and we showed in §\ref{sec:fragmentation} that selection failures are precisely cases where the correct mass is spread out (within-population $d = 2.98$ on top-token mass), so one might expect models that hallucinate more to also have higher $\Spread$. The resolution is that the across-scale decrease in SF-conditional $\Spread$ is a \emph{consequence} of wrong-token sharpening, not a cause of it: as the wrong token's probability rises (0.31 to 0.57 across Instruct scale), the correct concept's remaining mass within selection-failure samples is squeezed and necessarily concentrated on fewer effective surface forms. The within-population fragmentation effect (which holds across all 18 models, $|d| \geq 1.0$, Appendix~\ref{app:within_population_app}) is a separate phenomenon from the across-scale $\Spread$ trend. Llama Instruct sharpens earlier (the largest drop is from 1B to 3B), Qwen more gradually, but both converge near 1.1--1.2 at the largest scales. Note that this SF-conditional $\Spread$ is distinct from the unconditional $\Spread$ averaged over all samples (which is more sensitive to the bulk of low-$\Pmass$ samples), and from the within-population test on $\max_{v \in \Scstar} P_\theta(v)$ in Appendix~\ref{app:within_population_app} (which controls for $\Pmass$).

\textbf{Within-population comparison using $\Spread$ (Qwen3.5-9B Instruct, $\Pmass \geq 0.2$).} For completeness, we apply the within-population test of §\ref{sec:fragmentation} but with $\Spread$ in place of $\max_{v \in \Scstar} P_\theta(v)$:

\begin{center}
\small
\begin{tabular}{lccc}
\toprule
Group & $N$ & Mean $\Spread$ & Median $\Spread$ \\
\midrule
First-token selection failures & 128 & 1.31 & 1.10 \\
Correct samples (same $\Pmass$ range) & 840 & 1.16 & 1.02 \\
\bottomrule
\end{tabular}
\end{center}

Welch $t = 2.94$, $p = 3.4\!\times\!10^{-3}$, Cohen's $d = 0.32$. The direction matches expectation (failures have slightly higher $\Spread$ than correct samples) and is statistically significant, but the effect size is much smaller than the corresponding test on $\max P_\theta(v)$ ($d = 2.98$ in §\ref{sec:fragmentation}). The reason: $\max P_\theta(v)$ is a sharper signal for selection at the argmax level than $\Spread$ is---$\Spread$ averages over the whole alias distribution and treats, say, ``one alias token at 0.5 plus three at 0.05'' similarly to ``four alias tokens at 0.15 each'' (similar $\Spread$), even though only the first wins the argmax against a competing 0.4-probability token. The main-text analysis therefore uses $\max P_\theta(v)$ as the primary measure; we report $\Spread$ here for completeness.

\textbf{Within-belief stratification by $\Spread$ (commitment failures, all $\Pmass$).} Among hallucinated samples with $\Pmass \geq 0.2$ in Qwen3.5-9B Instruct (i.e., commitment failures, including both first-token selection failures and multi-token divergences):

\begin{center}
\small
\begin{tabular}{lccc}
\toprule
$\Pmass(t{=}1)$ band & $\Spread \in [1.0, 1.5]$ & $\Spread \in (1.5, 2.5]$ & $\Spread > 2.5$ \\
\midrule
$[0.2, 0.4)$ & 73.9\% ($N{=}199$) & \textbf{83.3\%} ($N{=}36$) & 76.5\% ($N{=}17$) \\
$[0.4, 0.6)$ & 58.6\% ($N{=}174$) & 62.5\% ($N{=}32$) & 50.0\% ($N{=}6$) \\
$[0.6, 0.8)$ & 43.3\% ($N{=}178$) & 56.5\% ($N{=}23$) & 0\% ($N{=}2$) \\
$[0.8, 1.0]$ & 31.2\% ($N{=}756$) & 29.5\% ($N{=}78$) & 50.0\% ($N{=}6$) \\
\bottomrule
\end{tabular}
\end{center}

A directional positive trend appears in three of four bands ($\sim$5--10pp), inconsistent in the highest bin where $N$ is small.

\clearpage
\section{Within-Population Effect Across Models}
\label{app:within_population_app}

For each model, we replicate the within-population test of §\ref{sec:fragmentation}: among samples with $\Pmass \geq 0.2$, compare $\max_{v \in \Scstar} P_\theta(v)$ between first-token selection failures (greedy $\notin \Scstar$, ``SF'') and correct samples in the same $\Pmass$ range. Table~\ref{tab:within_population} reports the comparison across all 18 models. The within-population effect ($d < 0$) is uniform: SF samples have lower top-token mass on $c^*$ than correct samples, in 100\% of models. Effect sizes range from $|d| = 1.01$ (Qwen3.5-4B Base, $N_\text{corr} = 12$) to $|d| = 4.30$ (Qwen2.5-72B Inst), with median $|d| = 1.93$.

\begin{table}[h]
\caption{Within-population effect at $\Pmass(t_c; c^*) \geq 0.2$, comparing $\max_{v \in \Scstar} P_\theta(v)$ between first-token selection failures (greedy $\notin \Scstar$, ``SF'') and correct samples in the same $\Pmass$ range. $d$: Cohen's $d$ on Top1, SF vs.\ Corr (negative means SF $<$ Corr, the expected direction). $p$: Welch's $t$-test $p$-value. The within-population effect is consistent across all 18 models ($d < 0$ in 100\%, $|d| \geq 1.0$, all $p < 10^{-2}$).}
\label{tab:within_population}
\centering
\small
\setlength{\tabcolsep}{6pt}
\begin{tabular}{lrrrr}
\toprule
Model & $N_{\text{SF}}$ & $N_{\text{Corr}}$ & $d$ & $p$ \\
\midrule
Qwen3.5-0.8B Inst  & 81  & 202 & $-1.34$ & $5.8 \times 10^{-28}$ \\
Qwen3.5-2B Inst    & 97  & 280 & $-1.88$ & $3.8 \times 10^{-64}$ \\
Qwen3.5-4B Inst    & 128 & 703 & $-2.79$ & $3.8 \times 10^{-199}$ \\
Qwen3.5-9B Inst    & 128 & 840 & $-2.98$ & $2.2 \times 10^{-186}$ \\
Qwen2.5-72B Inst   & 114 & 1{,}044 & $-4.30$ & $3.6 \times 10^{-133}$ \\
Llama-3.2-1B Inst  & 84  & 403 & $-2.18$ & $1.2 \times 10^{-89}$ \\
Llama-3.2-3B Inst  & 157 & 904 & $-2.73$ & $1.9 \times 10^{-233}$ \\
Llama-3.1-8B Inst  & 178 & 941 & $-2.80$ & $1.7 \times 10^{-229}$ \\
Llama-3.1-70B Inst & 170 & 1{,}288 & $-2.97$ & $5.4 \times 10^{-263}$ \\
\midrule
Qwen3.5-0.8B Base  & 62  & 34  & $-1.34$ & $3.8 \times 10^{-6}$ \\
Qwen3.5-2B Base    & 16  & 11  & $-1.28$ & $1.6 \times 10^{-2}$ \\
Qwen3.5-4B Base    & 34  & 12  & $-1.01$ & $3.3 \times 10^{-2}$ \\
Qwen3.5-9B Base    & 9   & 12  & $-1.45$ & $5.2 \times 10^{-3}$ \\
Qwen2.5-72B Base   & 119 & 988 & $-1.84$ & $7.2 \times 10^{-175}$ \\
Llama-3.2-1B Base  & 44  & 305 & $-1.64$ & $7.2 \times 10^{-38}$ \\
Llama-3.2-3B Base  & 78  & 728 & $-2.01$ & $2.7 \times 10^{-86}$ \\
Llama-3.1-8B Base  & 112 & 885 & $-1.99$ & $2.1 \times 10^{-126}$ \\
Llama-3.1-70B Base & 210 & 1{,}049 & $-1.47$ & $8.7 \times 10^{-173}$ \\
\bottomrule
\end{tabular}
\end{table}

\textbf{Effect-size patterns.} Two robust patterns emerge: (1)~within Instruct models, $|d|$ grows monotonically with scale (Qwen Inst: 1.34$\to$1.88$\to$2.79$\to$2.98$\to$4.30 from 0.8B to 72B; Llama Inst: 2.18$\to$2.73$\to$2.80$\to$2.97 from 1B to 70B); (2)~Instruct models show larger $|d|$ than size-matched Base models in nearly every case (e.g., 9B: 2.98 Inst vs.\ 1.45 Base; 70B+: 2.97--4.30 Inst vs.\ 1.47--1.84 Base; 0.8B is the one exception, with both at $|d|=1.34$). Statistical significance is overwhelming throughout: all 18 models give $p < 10^{-2}$, with $p < 10^{-130}$ for the 13 models with $N_\text{SF}, N_\text{Corr} \geq 100$. The Instruct--Base contrast in correct-sample top-token mass shows that the difference is one of distributional sharpness throughout, not specific to selection failures.

\clearpage
\section{Within-Concept and Between-Concept Mass Decomposition ($D_2$, $D_3$)}
\label{app:d2_d3}

This appendix supports the ``two faces of commitment failure'' analysis in §\ref{sec:fragmentation}. We measure two ratios on first-token selection failure (SF) samples:
\begin{align}
D_2 &= \frac{\max_{v \in \Scstar} P_\theta(v \mid Q)}{\Pmass(t_c; c^*)} \quad\text{(within-concept top-1 share)} \\
D_3 &= \frac{P_\theta(\text{greedy}) }{\Pmass(t_c; c^*)} \quad\text{(wrong-token dominance)}
\end{align}
$D_2$ measures how concentrated the correct concept's mass is on its single most-probable alias token: $D_2 \to 1$ means the mass has effectively collapsed onto one alias; lower $D_2$ indicates fragmentation across multiple alias forms. $D_3$ measures the wrong greedy token's mass relative to the correct concept's total: $D_3 > 1$ means the wrong token alone exceeds the entire correct concept.

\textbf{Two failure modes.} Together, $D_2$ and $D_3$ separate two structurally different mechanisms of commitment failure:
\begin{itemize}
\item \emph{Fragmentation-driven} (low $D_2$, low $D_3$): correct mass is spread across alias surface forms (e.g., Figure~\ref{fig:fragmentation}); no single correct alias is large enough to win, but neither is the wrong token strongly dominant.
\item \emph{Wrong-attractor-driven} (high $D_2$, high $D_3$): correct mass has collapsed onto a single alias, but a wrong concept's token is even sharper.
\end{itemize}

\textbf{Per-model statistics.} Table~\ref{tab:d2_d3} reports median $D_2$, median $D_3$, and the fraction of SF samples with $D_2 \geq 0.95$ (correct mass effectively collapsed) across all 18 models.

\begin{table}[h]
\caption{Within- and between-concept mass ratios for SF samples across 18 models. $D_2$ measures how much of the correct concept's mass is in its top-1 alias token. $D_3$ measures the wrong greedy token's mass relative to the entire correct concept. Within Instruct models, $D_2$ grows monotonically with scale and the high-$D_2$ fraction climbs from 28\% (1B) to 82\% (72B); $D_3$ similarly grows from 1.13 to 1.65. Within Base models, both quantities stay flat. Sharpening collapses fragmentation in Instruct but strengthens wrong attractors in parallel.}
\label{tab:d2_d3}
\centering
\small
\setlength{\tabcolsep}{4pt}
\begin{tabular}{llrcccc}
\toprule
Family & Variant & Size & $N_\text{SF}$ & median $D_2$ & median $D_3$ & frac $D_2 \geq 0.95$ \\
\midrule
\multirow{5}{*}{Qwen}  & \multirow{5}{*}{Inst} & 0.8B & 82  & 0.843 & 1.13 & 35.4\% \\
                       &                       & 2B   & 104 & 0.941 & 1.21 & 45.2\% \\
                       &                       & 4B   & 136 & 0.934 & 1.25 & 46.3\% \\
                       &                       & 9B   & 127 & 0.976 & 1.42 & 57.5\% \\
                       &                       & 72B  & 114 & \textbf{0.999} & \textbf{1.65} & \textbf{81.6}\% \\
\midrule
\multirow{5}{*}{Qwen}  & \multirow{5}{*}{Base} & 0.8B & 61  & 0.736 & 1.08 & 29.5\% \\
                       &                       & 2B   & 17  & 0.935 & 1.14 & 41.2\% \\
                       &                       & 4B   & 42  & 0.948 & 1.77 & 47.6\% \\
                       &                       & 9B   & 13  & 0.940 & 1.55 & 46.2\% \\
                       &                       & 72B  & 119 & 0.756 & 1.11 & 29.4\% \\
\midrule
\multirow{4}{*}{Llama} & \multirow{4}{*}{Inst} & 1B   & 98  & 0.759 & 1.13 & 28.6\% \\
                       &                       & 3B   & 164 & 0.982 & 1.28 & 61.6\% \\
                       &                       & 8B   & 178 & 0.966 & 1.41 & 55.1\% \\
                       &                       & 70B  & 170 & \textbf{0.990} & \textbf{1.45} & \textbf{66.5}\% \\
\midrule
\multirow{4}{*}{Llama} & \multirow{4}{*}{Base} & 1B   & 44  & 0.659 & 0.99 & 11.4\% \\
                       &                       & 3B   & 81  & 0.682 & 1.00 & 16.0\% \\
                       &                       & 8B   & 124 & 0.766 & 1.05 & 26.6\% \\
                       &                       & 70B  & 210 & 0.812 & 1.13 & 31.9\% \\
\bottomrule
\end{tabular}
\end{table}

\textbf{Patterns.} Within Llama, the contrast is clean: Instruct $D_2$ grows monotonically (0.76 $\to$ 0.98 $\to$ 0.97 $\to$ 0.99) while Base plateaus (0.66 $\to$ 0.68 $\to$ 0.77 $\to$ 0.81), widening the Inst--Base gap from $+0.10$ at 1B to $+0.18$ at 70B. The $D_2 \geq 0.95$ fraction climbs from 28.6\% (Llama-1B Inst) to 66.5\% (Llama-70B Inst), while Llama Base stays in 11.4\%--31.9\% across the entire range. Qwen2.5-72B shows the largest absolute Inst--Base gap ($+0.24$, Inst 1.00 vs Base 0.76) and confirms the same pattern at the largest scale; for Qwen3.5 Base in the 2B--9B range, $N_\text{SF}$ is small (13--42), making medians noisy estimates. $D_3$ tracks the same trend: monotonic growth in Instruct (Llama 1.13$\to$1.45; Qwen 1.13$\to$1.65) and flat in Llama Base (0.99--1.13).

\textbf{Connection to within-population $|d|$.} The within-population effect ($d = 2.98$ on Qwen-9B Inst, §\ref{sec:fragmentation}) measures the absolute top-1 alias probability difference between SF and matched-correct samples (0.26 vs.\ 0.78). $D_2$ normalizes this top-1 by $\Pmass(c^*)$ to isolate within-concept structure, revealing that at Qwen-9B Inst the SF samples' top-1 already accounts for $\sim$98\% of the correct concept's mass: the SF-Corr top-1 gap arises mostly from $\Pmass$ differences (SF samples have $\Pmass \approx 0.26$, correct samples $\approx 0.79$), not from fragmentation within the correct concept. This refines the §\ref{sec:fragmentation} narrative: in mid-to-large Instruct models, the structural difference between SF and matched-correct is the absolute mass on $c^*$ (and therefore on its top alias), with within-concept fragmentation playing a secondary role.

\textbf{Connection to direct decoding intervention.} Replacing greedy argmax with cluster-argmax over normalized top-50 tokens at $t = t_c$ recovers 5.7\% of SF samples in Llama-70B Base and 6.7\% in Qwen-72B Base, but only 2.4\% and 1.8\% in the corresponding Instruct models. The recovery rate is anti-correlated with $D_2$: when $D_2 \to 1$ there is no within-concept aggregation to do at decoding time. This confirms the mechanism dichotomy and supports the future-work direction (§\ref{sec:discussion}) of concept-aware decoding for non-Instruct or smaller models, where fragmentation is recoverable.

\clearpage
\section{Phrase-Level Commitment: $H_{t=2}$ Analysis on Multi-Token Divergences}
\label{app:t2_entropy}

This appendix supports the phrase-level sharpening claim in §\ref{sec:fragmentation}. We restrict to multi-token divergences---commitment failures where greedy $y_1 \in \Scstar$ but the final answer is wrong---and use the entropy of the next-token distribution $H_{t=2}$, conditioned on the emitted $y_1$, to probe whether commitment is finalized at $t=1$ (low $H_{t=2}$, deterministic continuation) or distributed across multiple tokens (high $H_{t=2}$, multiple candidate continuations).

\textbf{Type A vs.\ Type B classification.} For each multi-token divergence sample, we check whether the emitted bigram $(y_1, y_2)$ matches the start of any ground-truth alias of $c^*$ under the model's tokenizer. \emph{Type A}: bigram matches an alias prefix---the model is on a trajectory consistent with a valid surface form of $c^*$ at the first two tokens. \emph{Type B}: bigram does not match any alias---$y_1$ lies in $\Scstar$ but the second token already diverges from any valid $c^*$ realization. Type A is the signature of phrase-level commitment to a $c^*$-aligned phrase at $t = 1$.

A subtlety: Type A samples are by definition still hallucinations (final substring match against ground-truth aliases failed), so they are cases where the model's bigram aligns with an alias prefix but the multi-token completion still diverges into a different entity than $c^*$. Common patterns include sharing a personal name with an unrelated figure (\texttt{George Washington Carver} when the answer is \texttt{George Washington}), sharing a place-name prefix with a different geographic entity (\texttt{Saint Petersburg Beach} when the answer is \texttt{Saint Petersburg}), or sharing an entity prefix with a related-but-distinct concept (\texttt{New York Times} when the answer is \texttt{New York}). In each case, the bigram is consistent with a valid alias of $c^*$ but the continuation commits to a wrong concept. By contrast, \texttt{Adam Lambert} for \texttt{Adam Smith} is Type B---the bigram \texttt{Adam Lambert} matches no alias of \texttt{Adam Smith}.

\textbf{Per-model results.} Table~\ref{tab:t2_entropy} reports Type A fraction and $H_{t=2}$ statistics across 18 models. Two patterns are robust:

\begin{enumerate}[leftmargin=*,itemsep=0pt,topsep=2pt]
\item \textbf{Type A has substantially lower $H_{t=2}$ than Type B}, in 100\% of models (median Cohen's $d = 1.29$, range $[0.71, 1.87]$). When the bigram aligns with an alias prefix, the continuation is near-deterministic.
\item \textbf{Type A fraction grows with both scale and instruction tuning.} Small Base: 21--53\%. Small Instruct: 36--72\%. Large Base (70B+): 59--74\%. Large Instruct (70B+): \textbf{77--83\%}.
\end{enumerate}

\begin{table}[h]
\caption{Multi-token divergence diagnostic across 18 models. $N_{\text{Multi}}$: number of multi-token divergences. Type A frac.: fraction where the bigram $(y_1, y_2)$ matches a valid alias prefix. $H_{t=2}$ A / B: mean entropy at $t=2$ for Type A / Type B divergences. $d_{B-A}$: Cohen's $d$ comparing $H_{t=2}$ Type B vs.\ Type A. The $d_{B-A}$ for Qwen3.5-9B Base is undefined because $N = 4$ Type B samples is too small.}
\label{tab:t2_entropy}
\centering
\small
\setlength{\tabcolsep}{5pt}
\begin{tabular}{lrrrrr}
\toprule
Model & $N_{\text{Multi}}$ & Type A frac.\ & $H_{t=2}$ A & $H_{t=2}$ B & $d_{B-A}$ \\
\midrule
Qwen3.5-0.8B Inst   & 372 & 44\% & 0.45 & 3.19 & 1.87 \\
Qwen3.5-2B Inst     & 369 & 36\% & 0.57 & 2.63 & 1.69 \\
Qwen3.5-4B Inst     & 461 & 67\% & 0.35 & 1.97 & 1.20 \\
Qwen3.5-9B Inst     & 539 & 72\% & 0.20 & 1.72 & 1.29 \\
Qwen2.5-72B Inst    & 663 & 77\% & 0.05 & 0.43 & 1.13 \\
Llama-3.2-1B Inst   & 317 & 59\% & 0.61 & 2.36 & 1.26 \\
Llama-3.2-3B Inst   & 419 & 72\% & 0.38 & 1.82 & 1.14 \\
Llama-3.1-8B Inst   & 479 & 64\% & 0.26 & 1.32 & 1.03 \\
Llama-3.1-70B Inst  & 610 & \textbf{83\%} & \textbf{0.10} & 0.59 & 1.31 \\
\midrule
Qwen3.5-0.8B Base   & 374 & 21\% & 1.40 & 2.94 & 1.23 \\
Qwen3.5-2B Base     & 21  & 38\% & 0.82 & 1.56 & 0.71 \\
Qwen3.5-4B Base     & 25  & 36\% & 0.43 & 1.97 & 1.40 \\
Qwen3.5-9B Base     & 4   & 75\% & 0.43 & 2.57 & --   \\
Qwen2.5-72B Base    & 706 & 74\% & 0.37 & 2.04 & 1.87 \\
Llama-3.2-1B Base   & 435 & 24\% & 1.82 & 3.67 & 1.32 \\
Llama-3.2-3B Base   & 508 & 40\% & 1.23 & 2.92 & 1.33 \\
Llama-3.1-8B Base   & 435 & 53\% & 0.78 & 2.43 & 1.53 \\
Llama-3.1-70B Base  & 531 & 59\% & 0.67 & 2.07 & 1.48 \\
\bottomrule
\end{tabular}
\end{table}

\textbf{Pooled meta-analysis.} Across all 18 models, the Type B vs.\ Type A $H_{t=2}$ comparison gives a median Cohen's $d = 1.29$, with $d > 0$ in 100\% of models with sufficient $N$ and a Fisher-combined $p < 10^{-200}$. Phrase-level commitment is real and pervasive: when the bigram aligns with an alias prefix (Type A), the continuation is near-deterministic; when it does not (Type B), the second token retains substantial entropy across multiple candidates.

\textbf{Interpretation.} The 70B+ Instruct numbers are striking: Llama-3.1-70B Instruct has $H_{t=2} = 0.10$ on Type A divergences and Qwen2.5-72B Instruct has $H_{t=2} = 0.05$. These values approach the entropy of deterministic continuations---the model has committed to a specific multi-token phrase already at $t = 1$, with the second token essentially predetermined. This is the phrase-level analog of the first-token sharpening documented in §\ref{sec:fragmentation}: instruction-induced sharpening operates not just at the token level but across multi-token phrase commitments, and it grows monotonically with scale.

\textbf{Caveat on $\Pmass$ interpretation.} Within multi-token divergences, $\Pmass(t=1; c^*) \geq 0.2$ is consistent with two distributional pictures: (i)~the model placed mass on $c^*$-aligned alias prefixes at $t=1$ as part of a phrase-level commitment to a specific multi-token continuation (Type A); (ii)~the model placed mass on alias first tokens that are also the start of competing concepts' phrases (Type B---e.g., generic \texttt{Sir}, \texttt{Saint} shared across many entities). $\Pmass$ does not distinguish these, but the within-population analysis in §\ref{sec:fragmentation} (which restricts to first-token selection failures) does not depend on this distinction.

\clearpage
\section{Robustness of CF\% to the 0.2 Threshold}
\label{app:knew_threshold}

The 0.2 threshold defining commitment failures is conservative and somewhat arbitrary. Table~\ref{tab:knew_thresh} reports CF\% at thresholds $\{0.1, 0.2, 0.3, 0.4\}$ across the scale ablation: the absolute level shifts but the monotonic increase with model size is preserved at all thresholds.

\begin{table}[h]
\caption{CF\% across thresholds $\theta \in \{0.1, 0.2, 0.3, 0.4\}$ defining commitment failures as hallucinated samples with $\Pmass(t_c; c^*) \geq \theta$. The absolute level shifts with threshold, but the monotonic increase with model size is preserved at every column.}
\label{tab:knew_thresh}
\centering
\small
\begin{tabular}{llcccc}
\toprule
Family & Model & $\theta = 0.1$ & $\theta = 0.2$ & $\theta = 0.3$ & $\theta = 0.4$ \\
\midrule
\multirow{4}{*}{Qwen3.5}
 & 0.8B Inst & 28.2\% & 16.5\% & 10.1\% & 6.5\% \\
 & 2B Inst   & 28.9\% & 17.4\% & 12.3\% & 9.4\% \\
 & 4B Inst   & 37.2\% & 26.0\% & 20.1\% & 16.7\% \\
 & 9B Inst   & 41.7\% & 31.5\% & 26.0\% & 22.5\% \\
\midrule
\multirow{3}{*}{Llama 3.2/3.1}
 & 1B Inst   & 26.8\% & 15.7\% & 10.4\% & 6.9\% \\
 & 3B Inst   & 37.9\% & 28.2\% & 22.4\% & 17.8\% \\
 & 8B Inst   & 41.2\% & 32.6\% & 27.1\% & 22.9\% \\
\midrule
\multirow{2}{*}{Qwen2.5}
 & 72B Inst & 44.3\% & 40.7\% & 38.1\% & 35.6\% \\
 & 72B Base & 54.9\% & 41.8\% & 34.4\% & 27.4\% \\
\midrule
\multirow{2}{*}{Llama 3.1}
 & 70B Inst & 53.6\% & 47.0\% & 42.0\% & 37.4\% \\
 & 70B Base & 55.7\% & 39.5\% & 27.0\% & 20.2\% \\
\bottomrule
\end{tabular}
\end{table}

\clearpage
\section{Long-Form Generation: $\Pmass$ Trajectories}
\label{app:long_form}

This appendix supports the discussion in §\ref{sec:eps_align} on how $\Pmass$ behaves when the commitment step is not at $t = 1$. We use Qwen3.5-9B Instruct on TriviaQA + NQ-Open and re-prompt each question with a long-form instruction (\texttt{Answer the following question in a complete sentence.}). The model now produces an answer like \texttt{The capital of France is Paris.} rather than \texttt{Paris}. The commitment step $t_c$ is identified manually as the position of the answer entity within the generated sentence (e.g., \texttt{Paris} in the example above), and we align trajectories around $t_c$. We pre-screen samples to those whose long-form output contains a unique unambiguous reference to a single concept (correct or wrong) to make $t_c$ well-defined.

Figure~\ref{fig:long_form} shows the resulting $\Pmass(t; c^*)$ trajectories. Correct samples have a sharp peak in $\Pmass$ at $t_c$ (typically 0.6--0.9) and near-zero mass before and after, consistent with $\Pmass$ measuring the model's commitment to the correct concept at the moment of emission. Hallucinated samples sit near zero at all aligned steps---the model never put substantial mass on $c^*$ in the trajectory. This is the long-form analog of ``not-CF'' hallucinations: cases where the model truly does not know the answer, distinct from CF/SF samples in short-form QA. Long-form $\Pmass$-trajectories under the right prompting can therefore separate ``model never had it'' from ``model had it but emitted something else,'' but the cleaner signal in our paper comes from instruction-tuned short-form QA where $t_c = 1$ is fixed.

For comparison, Figure~\ref{fig:tokprob_aligned} shows the analogous trajectory of the generated-token probability $P(y_t)$. Both correct and hallucinated trajectories are at $\sim$0.85--0.95 throughout, with only a small dip at $t_c$. Token-level confidence carries little of the correct/hallucinated signal that $\Pmass$ reveals, which is the central methodological point of the paper transferred to the long-form setting: the relevant signal lives at the level of concept-grouped mass, not individual-token entropy.

\begin{figure}[h]
\centering
\includegraphics[width=\linewidth]{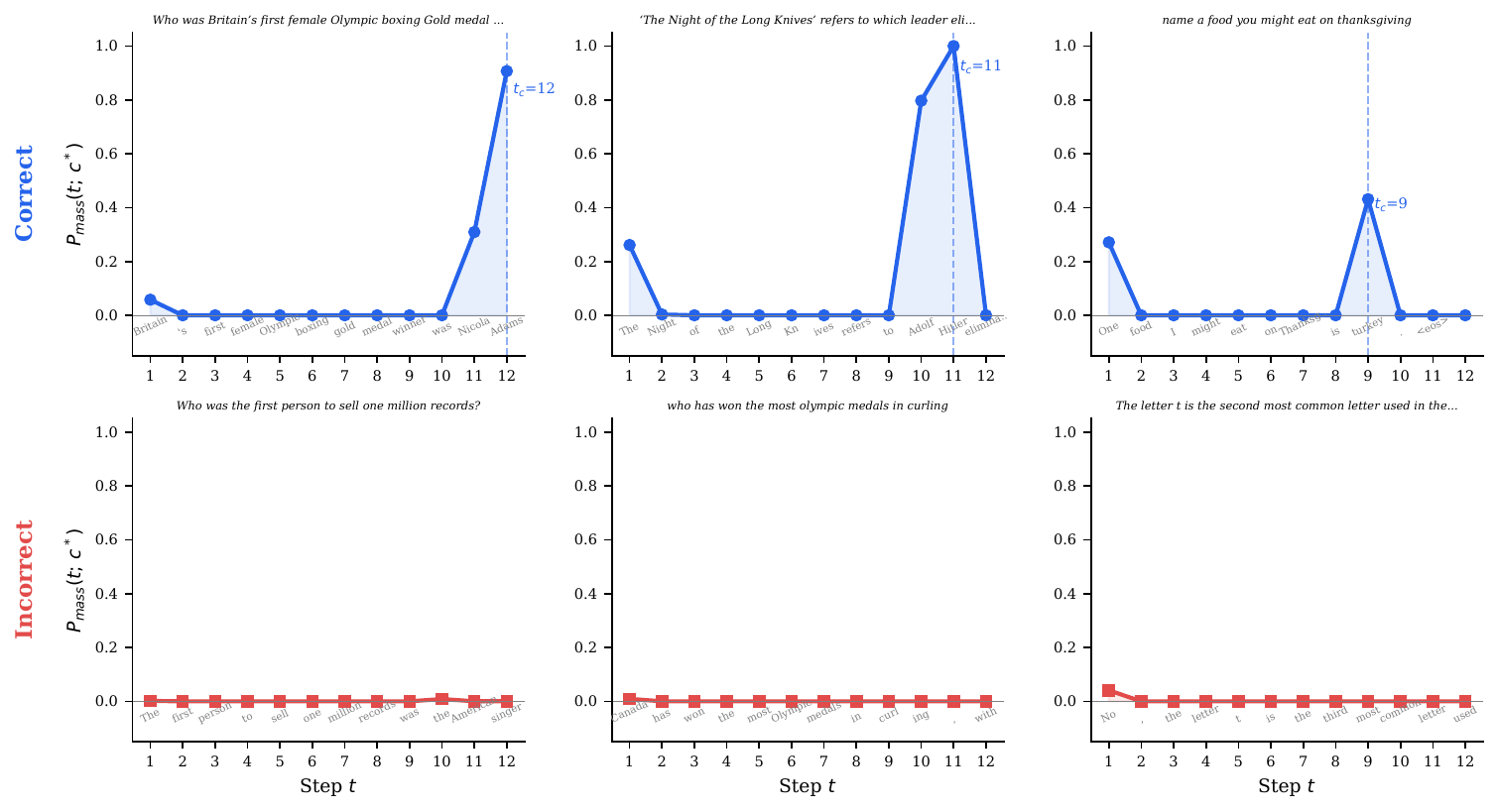}
\caption{$\Pmass(t; c^*)$ trajectories under long-form prompting (\texttt{Answer in a complete sentence}), Qwen3.5-9B Instruct. Top: correct samples; bottom: hallucinated. Dashed lines mark $t_c$ (the position of the answer entity within the generated sentence). Correct samples: $\Pmass$ is near-zero before $t_c$, peaks at $t_c$, collapses afterward. Hallucinated samples: essentially no mass on $c^*$ at any step.}
\label{fig:long_form}
\end{figure}

\begin{figure}[h]
\centering
\includegraphics[width=0.7\linewidth]{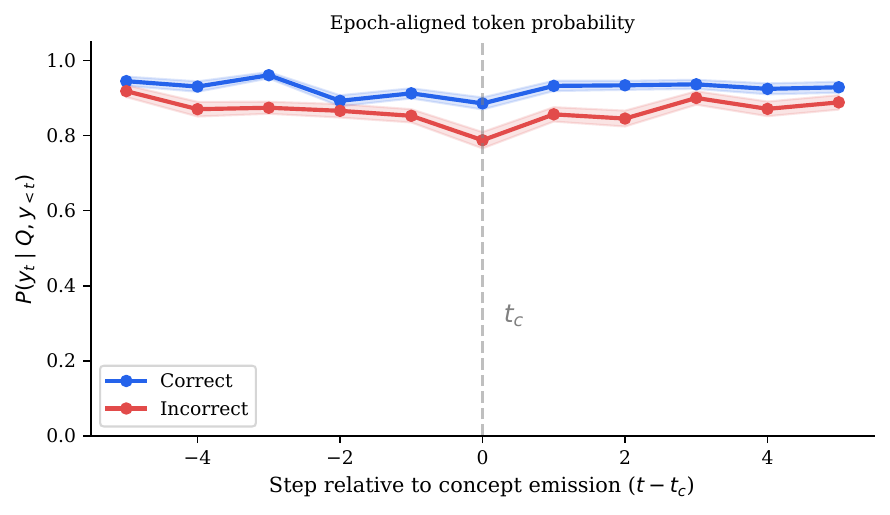}
\caption{Generated-token probability $P(y_t)$ aligned to $t_c$ in long-form generation. Both correct and hallucinated trajectories sit at $\sim$0.85--0.95 throughout, with only a small dip at $t_c$ ($\sim$0.78 vs.\ 0.89). Token-level confidence carries little of the correct/hallucinated signal that $\Pmass$ reveals (cf.\ Figure~\ref{fig:epoch_aligned}b).}
\label{fig:tokprob_aligned}
\end{figure}

\clearpage
\section{$\Sc$ Construction}
\label{app:sc_construction}

$\Sc$ is constructed deterministically with no LLM involvement:
\begin{itemize}[leftmargin=*,itemsep=0pt,topsep=2pt]
    \item \textbf{Alias collection.} TriviaQA: \texttt{answer.value}, \texttt{answer.aliases}, \texttt{answer.normalized\_aliases}. NQ-Open: all entries in \texttt{answer}. Typically 5--20 aliases per question.
    \item \textbf{Lexical variants.} For each alias $a$, six variants: original, lowercase, capitalized, and the same three with a leading space. We exclude \texttt{upper()} variants (single capital letters appear in many unrelated $\Sc$) and newline-prefixed variants (the \texttt{\textbackslash n} token appears in all $\Sc$).
    \item \textbf{First-token extraction.} Each variant is tokenized with \texttt{add\_special\_tokens=False}; the first token ID is added to $\Sc$.
    \item \textbf{Deduplication.} Final $|\Sc|$ is typically 12--20.
\end{itemize}

\clearpage
\section{Multi-Layer Probing}
\label{app:multilayer}

\begin{table}[h]
\caption{Multi-layer probe AUROC (MCQA, logistic regression, 5-fold CV).}
\label{tab:multilayer}
\centering
\small
\begin{tabular}{lcccc}
\toprule
Model & First & Mid & Last$-$1 & Last \\
\midrule
0.8B Inst & 0.565 & 0.624 & \textbf{0.697} & 0.637 \\
0.8B Base & 0.521 & 0.503 & 0.532 & 0.602 \\
\midrule
2B Inst   & 0.601 & \textbf{0.718} & 0.709 & 0.716 \\
2B Base   & 0.513 & 0.553 & \textbf{0.651} & 0.691 \\
\midrule
4B Inst   & 0.690 & \textbf{0.838} & 0.827 & 0.798 \\
4B Base   & 0.653 & 0.742 & \textbf{0.774} & 0.750 \\
\midrule
9B Inst   & 0.736 & 0.848 & \textbf{0.867} & 0.835 \\
9B Base   & 0.726 & 0.747 & 0.776 & \textbf{0.844} \\
\bottomrule
\end{tabular}
\end{table}

\begin{figure}[h]
\centering
\includegraphics[width=\linewidth]{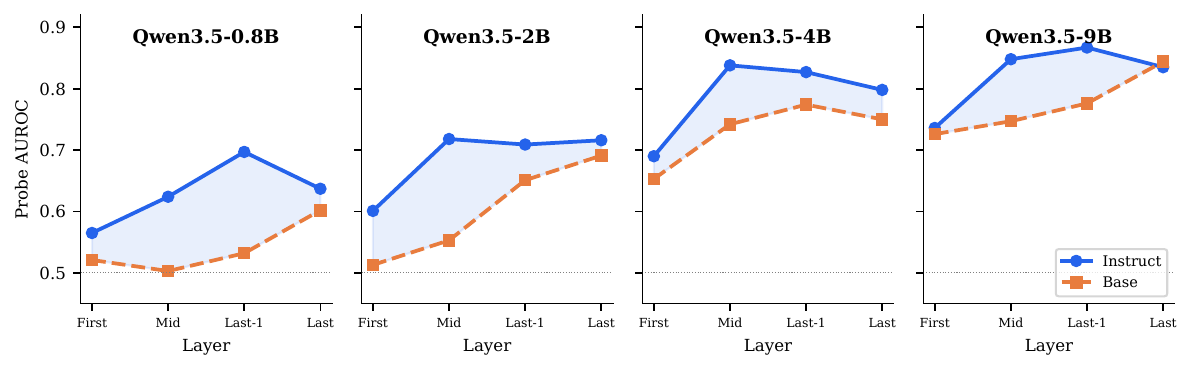}
\caption{Layer-wise probe AUROC (Qwen3.5). Instruct (blue) $>$ Base (orange) at nearly every layer; gap peaks at mid-layers.}
\label{fig:layerwise}
\end{figure}

The first-layer Instruct $>$ Base difference (+0.01 to +0.04) rules out a purely output-formatting explanation. Mid-layer peaks ($\approx +0.06$ to +0.12 gap) indicate that instruction tuning amplifies correctness encoding most strongly in intermediate representations.

\clearpage
\section{Pre- vs.\ Post-Generation Probes}
\label{app:postgen}

\begin{table}[h]
\caption{Pre-gen vs.\ post-gen probe AUROC (MCQA, last layer, $N{=}1{,}000$).}
\label{tab:postgen_app}
\centering
\small
\begin{tabular}{lcccccc}
\toprule
 & \multicolumn{3}{c}{Instruct} & \multicolumn{3}{c}{Base} \\
\cmidrule(lr){2-4} \cmidrule(lr){5-7}
Size & Pre & Post & $\Delta_I$ & Pre & Post & $\Delta_B$ \\
\midrule
0.8B & 0.637 & 0.626 & $-$0.011 & 0.602 & 0.610 & +0.008 \\
2B   & 0.716 & 0.730 & +0.015 & 0.691 & \textbf{0.771} & \textbf{+0.080} \\
4B   & 0.798 & 0.819 & +0.021 & 0.750 & 0.784 & +0.034 \\
9B   & 0.835 & 0.830 & $-$0.005 & 0.844 & 0.843 & $-$0.001 \\
\midrule
Avg & & & +0.005 & & & +0.030 \\
\bottomrule
\end{tabular}
\end{table}

\clearpage
\section{Extended Tables}
\label{app:extended_results}

\begin{table}[h]
\caption{First-token behavior at $t = 1$ (Short-QA). Instruct: $t_c = 1$; Base: $t_c \geq 2$ due to a leading filler.}
\label{tab:firsttoken}
\centering
\small
\begin{tabular}{llcc}
\toprule
Model & Top-1 token $y_1$ & $P_\theta(y_1 \mid Q)$ & $t_c$ \\
\midrule
0.8B Inst & answer / \texttt{The} & 30--83\% & 1 \\
2B Inst   & answer / \texttt{The} & 39--60\% & 1 \\
4B Inst   & answer & 57--83\% & 1 \\
9B Inst   & answer & 52--82\% & 1 \\
\midrule
0.8B Base & \texttt{\textbackslash n\textbackslash n} & 24--84\% & $\geq$1 \\
2B Base   & \texttt{\textbackslash n\textbackslash n} & 93--98\% & $\geq$2 \\
4B Base   & \texttt{\textbackslash n\textbackslash n} & 96--99\% & $\geq$2 \\
9B Base   & \texttt{\textbackslash n\textbackslash n} & 45--66\% & $\geq$1 \\
\bottomrule
\end{tabular}
\end{table}

\begin{table}[h]
\caption{Full scale ablation (Short-QA, $N{=}3{,}000$). TokP / $\Pmass$: AUROC for generated-token probability and $\Pmass(t{=}1)$. $\Spread$: average $\Spread(c^*)$. CF\%: commitment-failure rate.}
\label{tab:scale_full}
\centering
\small
\setlength{\tabcolsep}{4pt}
\begin{tabular}{llccccc}
\toprule
Family & Model & Acc & TokP & $\Pmass$ & $\Spread$ & CF\% \\
\midrule
\multirow{8}{*}{Qwen3.5}
 & 0.8B Inst & 8.3\%  & .759 & .898 & 1.91 & 16\% \\
 & 2B Inst   & 11.0\% & .744 & .893 & 1.43 & 17\% \\
 & 4B Inst   & 24.6\% & .832 & .912 & 1.53 & 26\% \\
 & 9B Inst   & 29.4\% & .806 & .887 & 1.39 & 32\% \\
 & 0.8B Base & 1.8\%  & .524 & .902 & 1.71 & 15\% \\
 & 2B Base   & 1.0\%  & .307 & .881 & 1.46 & 1\% \\
 & 4B Base   & 2.0\%  & .296 & .833 & 1.40 & 2\% \\
 & 9B Base   & 3.0\%  & .318 & .795 & 1.21 & 0.4\% \\
\midrule
\multirow{3}{*}{Llama 3.2/3.1 Inst}
 & 1B Inst   & 14.8\% & .735 & .935 & 1.83 & 16\% \\
 & 3B Inst   & 31.4\% & .776 & .902 & 1.30 & 28\% \\
 & 8B Inst   & 32.8\% & .740 & .882 & 1.34 & 33\% \\
\midrule
\multirow{3}{*}{Llama 3.2/3.1 Base}
 & 1B Base   & 13.1\% & .755 & .887 & 2.09 & 18\% \\
 & 3B Base   & 26.8\% & .781 & .880 & 2.05 & 27\% \\
 & 8B Base   & 31.7\% & .812 & .899 & 1.73 & 27\% \\
\midrule
\multirow{2}{*}{Qwen2.5}
 & 72B Inst & 36.4\% & .705 & .830 & 1.11 & 41\% \\
 & 72B Base & 34.3\% & .776 & .848 & 1.78 & 42\% \\
\midrule
\multirow{2}{*}{Llama 3.1}
 & 70B Inst & 44.7\% & .719 & .816 & 1.18 & \textbf{47\%} \\
 & 70B Base & 37.5\% & .745 & .845 & 1.76 & 40\% \\
\bottomrule
\end{tabular}
\end{table}

All 18 models (14 small + 4 large) now have full metric coverage; CF\% values are verified against the data dump in §\ref{sec:selection}.

\begin{table}[h]
\caption{$\Pmass(t = 1)$ vs.\ generated-token probability (AUROC) with calibration. $\Pmass$ recovers a coherent confidence signal across all 14 models; this is the prerequisite (not the headline) for the commitment-failure analysis.}
\label{tab:epistemic_app}
\centering
\small
\begin{tabular}{llcccccc}
\toprule
Family & Model & Acc & TokProb & $\Pmass$ & $\Delta$ & ECE & Brier \\
\midrule
\multirow{3}{*}{Llama Inst}
 & 1B & 14.8\% & .735 & \textbf{.935} & +.200 & .023 & .070 \\
 & 3B & 31.4\% & .776 & \textbf{.902} & +.126 & .065 & .124 \\
 & 8B & 32.8\% & .740 & \textbf{.882} & +.142 & .080 & .144 \\
\midrule
\multirow{4}{*}{Qwen Inst}
 & 0.8B & 8.3\%  & .759 & \textbf{.898} & +.139 & .048 & .064 \\
 & 2B   & 11.0\% & .744 & \textbf{.893} & +.149 & .055 & .084 \\
 & 4B   & 24.6\% & .832 & \textbf{.912} & +.080 & .079 & .116 \\
 & 9B   & 29.4\% & .806 & \textbf{.887} & +.081 & .096 & .144 \\
\midrule
\multirow{3}{*}{Llama Base}
 & 1B & 13.1\% & .755 & \textbf{.887} & +.132 & .024 & .085 \\
 & 3B & 26.8\% & .781 & \textbf{.880} & +.099 & .037 & .125 \\
 & 8B & 31.7\% & .812 & \textbf{.899} & +.087 & .046 & .122 \\
\midrule
\multirow{4}{*}{Qwen Base}
 & 0.8B & 1.8\% & .524 & \textbf{.902} & +.378 & .045 & .026 \\
 & 2B   & 1.0\% & .307 & \textbf{.881} & +.574 & .006 & .009 \\
 & 4B   & 2.0\% & .296 & \textbf{.833} & +.536 & .011 & .018 \\
 & 9B   & 3.0\% & .318 & \textbf{.795} & +.476 & .020 & .025 \\
\bottomrule
\end{tabular}
\end{table}

\begin{figure}[h]
\centering
\includegraphics[width=\linewidth]{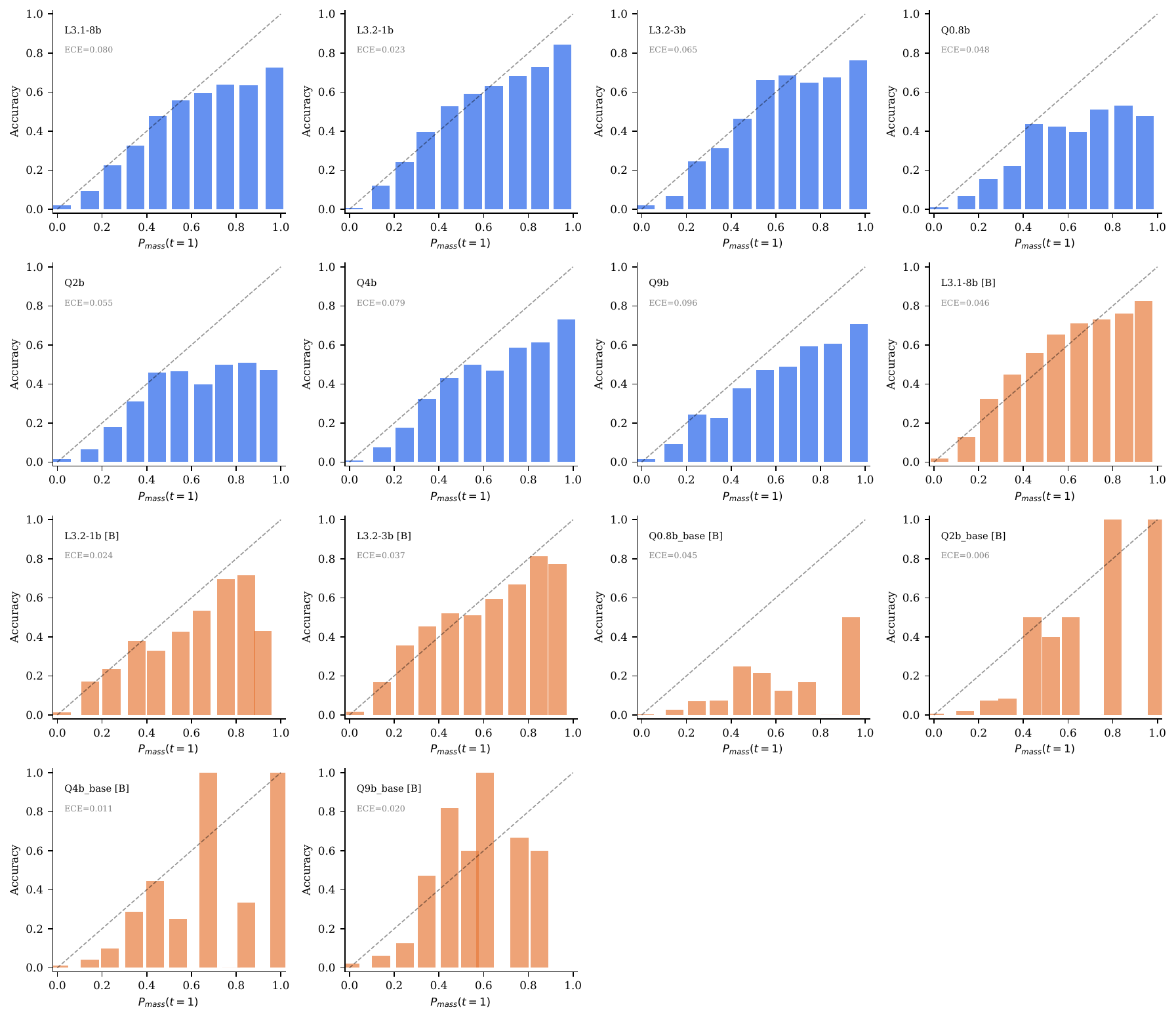}
\caption{$\Pmass(t = 1)$ calibration for all 14 models. Accuracy increases monotonically across $\Pmass$ bins (Instruct ECE 0.023--0.096).}
\label{fig:calibration_app}
\end{figure}

\clearpage
\section{Aggregation: Probability Space vs.\ Log Space}
\label{app:aggregation}

A natural question is how $\Pmass$ relates to standard sequence-level uncertainty estimators that aggregate per-step quantities across the full trajectory, typically in log space (mean log-probability, length-normalized NLL). For multi-token concepts these correspond to two different ways of forming a per-sequence score:

\begin{itemize}[leftmargin=*,itemsep=0pt,topsep=2pt]
    \item \emph{Probability-space first-step:} $\Pmass(t = 1; c^*)$, our default.
    \item \emph{Log-space full sequence:} $\frac{1}{T}\sum_{t=1}^{T} \log P_\theta(y_t \mid Q, y_{<t})$, the standard length-normalized log-likelihood (LN-NLL).
\end{itemize}

These differ in two ways: probability vs.\ log space, and single-step vs.\ trajectory-averaged. The two differences combine to give a sharp empirical contrast. On Qwen3.5-9B Instruct Short-QA ($N{=}3{,}000$), $\Pmass(t = 1; c^*)$ achieves AUROC 0.887 while LN-NLL achieves 0.806 (the same gap as $\Pmass$ vs.\ TokProb in Table~\ref{tab:epistemic_app}). Decomposing the gap by changing one factor at a time:

\begin{center}
\small
\begin{tabular}{lc}
\toprule
Estimator & AUROC \\
\midrule
$\Pmass(t = 1; c^*)$ (prob-space, single-step) & 0.887 \\
Average of per-step $\Pmass$ across full sequence (prob-space, averaged) & 0.738 \\
$\log P_\theta(y_1 \mid Q)$ (log-space, single-step) & 0.812 \\
LN-NLL (log-space, full sequence) & 0.806 \\
\bottomrule
\end{tabular}
\end{center}

Two observations. First, the single-step $\to$ full-sequence drop is large in probability space (0.887 $\to$ 0.738) but small in log space (0.812 $\to$ 0.806): aggregating downstream tokens dilutes the probability-space signal because most steps are deterministic continuations whose $\Pmass$ is essentially zero. Second, at the commitment step itself, probability space ($\Pmass$) outperforms log space ($\log P$) by 0.075 AUROC, because the relevant quantity at the commitment step is the total mass on the correct concept's surface forms---which is additive in probability, not in log-probability. The same picture holds for the relationship between greedy token probability and its log form: the probability is what is being compared by the argmax, so it is the natural quantity to inspect at the moment of commitment. We use probability space throughout the paper for this reason.

\clearpage
\section{Compute Resources}
\label{app:compute}

All experiments were run on NVIDIA GPUs. We did not perform any model fine-tuning---all compute consists of forward passes only. Per-model wall-clock cost scales with model size:

\begin{itemize}
\item \textbf{Sub-10B models} (Qwen3.5-0.8B/2B/4B/9B, Llama-3.2-1B/3B, Llama-3.1-8B): single NVIDIA A100 (80GB), fp16, batch size 8; 0.5--2 hours per (model, dataset) combination.
\item \textbf{72B/70B models} (Qwen2.5-72B, Llama-3.1-70B): single NVIDIA B200 (180GB), fp16, batch size 8; 4--8 hours per (model, dataset) combination. The B200's larger memory accommodates the 70B+ models without tensor parallelism.
\end{itemize}

The full 18-model evaluation across TriviaQA + NQ-Open (3{,}000 samples per dataset, top-50 token probabilities saved at the commitment step) totals approximately 53{,}000 forward passes. The Phase~2 within-concept and between-concept ratio analysis ($D_2$, $D_3$ in Appendix~\ref{app:d2_d3}) is a pure offline post-processing of the saved top-50 probabilities and required no additional GPU compute.

The probing experiments (Appendix~\ref{app:multilayer}) use logistic regression on saved hidden states; each probe trains in under one minute on CPU.

\clearpage
\section{Broader Impacts}
\label{app:broader_impacts}

This paper provides an analytical characterization of LLM hallucination structure: when a model places non-trivial mass on the correct concept yet emits a wrong final answer. The work introduces no new models, datasets, or deployable methods; the contribution is a probe and an empirical characterization. Potential positive societal impact: a clearer mechanistic account of confident hallucination may inform safer deployment and calibration practices, particularly in high-stakes settings where users treat fluency as evidence of reliability. We do not foresee direct negative societal impact from the analysis itself, beyond the general concern that deeper understanding of model failures could in principle inform adversarial exploitation; however, the analytical probe itself is not a generation or control method and does not enable any new attack surface.

\clearpage
\section*{NeurIPS Paper Checklist}

\begin{enumerate}

\item {\bf Claims}
    \item[] Question: Do the main claims made in the abstract and introduction accurately reflect the paper's contributions and scope?
    \item[] Answer: \answerYes{}
    \item[] Justification: The abstract and introduction state that we (i) introduce per-step semantic probability mass $\Pmass$ as an analytical probe over equivalence classes of token completions, (ii) show that 16--47\% of Instruct hallucinations occur with substantial mass on the correct concept and that the rate rises monotonically with scale, and (iii) characterize the underlying mechanism as instruction-induced sharpening at three structural levels. Each claim is empirically supported in §\ref{sec:results} (commitment-failure rates in Table~\ref{tab:cf_by_scale}; sharpening at three levels in §\ref{sec:fragmentation}; pre-generation hidden-state signal in §\ref{sec:frontload}).

\item {\bf Limitations}
    \item[] Question: Does the paper discuss the limitations of the work performed by the authors?
    \item[] Answer: \answerYes{}
    \item[] Justification: §\ref{sec:discussion} (Limitations paragraph) discusses (i) the dependence of $\Pmass$ on $\Sc$ construction and the implications when alias completeness is imperfect, (ii) the restriction to first-token commitment and the migration of $t_c$ in long-form generation, (iii) the analytical-probe (not detector) status of $\Pmass$, and (iv) the model scope (Qwen and Llama families up to 72B; no closed-source or non-English models tested).

\item {\bf Theory assumptions and proofs}
    \item[] Question: For each theoretical result, does the paper provide the full set of assumptions and a complete (and correct) proof?
    \item[] Answer: \answerYes{}
    \item[] Justification: The paper has one formal result, Proposition~\ref{prop:pmass}, relating $\Pmass$ to concept belief under a latent-concept generation model. Assumptions (alias completeness and concept separation) are stated explicitly in §\ref{sec:setup}; the full proof is in Appendix~\ref{app:proofs}.

\item {\bf Experimental result reproducibility}
    \item[] Question: Does the paper fully disclose all the information needed to reproduce the main experimental results of the paper to the extent that it affects the main claims and/or conclusions of the paper (regardless of whether the code and data are provided or not)?
    \item[] Answer: \answerYes{}
    \item[] Justification: All 18 models are publicly available on HuggingFace (Qwen3.5/Qwen2.5 \citep{qwen2025qwen3, qwen2024qwen25}, Llama-3.2/3.1 \citep{llama2024llama3}); all four datasets (TriviaQA, NQ-Open, MMLU, ARC-Challenge) are public. Sample counts, prompts, decoding settings (greedy), $\Scstar$ construction (Appendix~\ref{app:sc_construction}), and metric definitions are specified in §\ref{sec:setup}.

\item {\bf Open access to data and code}
    \item[] Question: Does the paper provide open access to the data and code, with sufficient instructions to faithfully reproduce the main experimental results, as described in supplemental material?
    \item[] Answer: \answerYes{}
    \item[] Justification: Anonymized code is provided in the supplementary material as a zip archive, including scripts to reproduce the main results (commitment-failure rate computation, within-population analysis, $H_{t=2}$ classification, $D_2$/$D_3$ measurements). All datasets are publicly available (TriviaQA, NQ-Open, MMLU, ARC-Challenge); models are publicly released on HuggingFace.

\item {\bf Experimental setting/details}
    \item[] Question: Does the paper specify all the training and test details (e.g., data splits, hyperparameters, how they were chosen, type of optimizer) necessary to understand the results?
    \item[] Answer: \answerYes{}
    \item[] Justification: §\ref{sec:setup} specifies all 18 models with sizes and variants, datasets and sample counts (3{,}000 short-QA per model; 2{,}672 MCQA per model), decoding (greedy), quantization (4-bit NF4 for Qwen3.5 sub-10B; fp16 for the rest), probe protocol (5-fold CV logistic regression), and metric definitions (AUROC, ECE, Brier).

\item {\bf Experiment statistical significance}
    \item[] Question: Does the paper report error bars suitably and correctly defined or other appropriate information about the statistical significance of the experiments?
    \item[] Answer: \answerYes{}
    \item[] Justification: Within-population comparisons report Welch's $t$-test statistics, $p$-values, and Cohen's $d$ effect sizes for all 18 models (§\ref{sec:fragmentation}, Appendix Table~\ref{tab:within_population}); the $H_{t=2}$ multi-token analysis reports Cohen's $d$ across the same 18 models (Appendix~\ref{app:t2_entropy}); a primer on these statistics is in Appendix~\ref{app:stats_primer}. All experiments use deterministic greedy decoding, so there is no run-to-run sampling variance to capture with confidence intervals; the variability captured by Welch's $t$ is the relevant within-sample distribution variability.

\item {\bf Experiments compute resources}
    \item[] Question: For each experiment, does the paper provide sufficient information on the computer resources (type of compute workers, memory, time of execution) needed to reproduce the experiments?
    \item[] Answer: \answerYes{}
    \item[] Justification: Compute resources are described in Appendix~\ref{app:compute} (single A100 80GB for sub-10B models; single B200 180GB for 70B+ models; all forward-pass-only experiments).

\item {\bf Code of ethics}
    \item[] Question: Does the research conducted in the paper conform, in every respect, with the NeurIPS Code of Ethics \url{https://neurips.cc/public/EthicsGuidelines}?
    \item[] Answer: \answerYes{}
    \item[] Justification: The work analyzes publicly released LLMs on standard public QA benchmarks. No human subjects, no scraped data, no new models or datasets released.

\item {\bf Broader impacts}
    \item[] Question: Does the paper discuss both potential positive societal impacts and negative societal impacts of the work performed?
    \item[] Answer: \answerYes{}
    \item[] Justification: Potential positive and negative societal impacts are discussed in Appendix~\ref{app:broader_impacts}.

\item {\bf Safeguards}
    \item[] Question: Does the paper describe safeguards that have been put in place for responsible release of data or models that have a high risk for misuse (e.g., pre-trained language models, image generators, or scraped datasets)?
    \item[] Answer: \answerNA{}
    \item[] Justification: No new models, datasets, or scraped data are released.

\item {\bf Licenses for existing assets}
    \item[] Question: Are the creators or original owners of assets (e.g., code, data, models), used in the paper, properly credited and are the license and terms of use explicitly mentioned and properly respected?
    \item[] Answer: \answerYes{}
    \item[] Justification: Models used: Qwen3.5/Qwen2.5 (Apache 2.0; \citep{qwen2025qwen3, qwen2024qwen25}) and Llama-3.1/3.2 (Llama Community License; \citep{llama2024llama3}). Datasets: TriviaQA \citep{joshi2017triviaqa}, NQ-Open \citep{kwiatkowski2019natural}, MMLU \citep{hendrycks2021measuring}, ARC-Challenge \citep{clark2018think}, all distributed under permissive academic-use licenses.

\item {\bf New assets}
    \item[] Question: Are new assets introduced in the paper well documented and is the documentation provided alongside the assets?
    \item[] Answer: \answerNA{}
    \item[] Justification: No new models, datasets, or other assets are released.

\item {\bf Crowdsourcing and research with human subjects}
    \item[] Question: For crowdsourcing experiments and research with human subjects, does the paper include the full text of instructions given to participants and screenshots, if applicable, as well as details about compensation (if any)?
    \item[] Answer: \answerNA{}
    \item[] Justification: No crowdsourcing or human subjects.

\item {\bf Institutional review board (IRB) approvals or equivalent for research with human subjects}
    \item[] Question: Does the paper describe potential risks incurred by study participants, whether such risks were disclosed to the subjects, and whether Institutional Review Board (IRB) approvals (or an equivalent approval/review based on the requirements of your country or institution) were obtained?
    \item[] Answer: \answerNA{}
    \item[] Justification: No human subjects.

\item {\bf Declaration of LLM usage}
    \item[] Question: Does the paper describe the usage of LLMs if it is an important, original, or non-standard component of the core methods in this research?
    \item[] Answer: \answerNA{}
    \item[] Justification: LLMs are the subject of study but are not used as a methodological component for core method development.

\end{enumerate}

\end{document}